\documentclass[10pt,twocolumn,letterpaper]{article}

\usepackage{iccv}
\usepackage{times}
\usepackage{epsfig}
\usepackage{graphicx,graphics}
\usepackage{amsmath}
\usepackage{amssymb}
\usepackage{booktabs}

\usepackage{caption}
\usepackage{algorithm}
\usepackage{algorithmic}
\usepackage{framed}
\usepackage{arydshln}
\usepackage{multirow}
\usepackage{multicol}
\usepackage[american]{babel}
\usepackage{microtype}
\usepackage{booktabs,bigstrut}
\usepackage{enumitem}
\usepackage{amsthm}
\usepackage{color}

\graphicspath{{Figure/}}

\usepackage{color, xcolor, colortbl}
\definecolor{citecolor}{HTML}{229954}
\usepackage[pagebackref=true,breaklinks=true,colorlinks,citecolor=citecolor,bookmarks=false]{hyperref}

\newcommand{\red}[1]{\textcolor{red}{#1}}

\newcommand{\bfsection}[1]{\vspace*{0.1cm}\noindent\textbf{#1.}}

\usepackage[capitalize]{cleveref}
\crefname{section}{Sec.}{Secs.}
\Crefname{section}{Section}{Sections}
\Crefname{table}{Table}{Tables}
\crefname{table}{Tab.}{Tabs.}

\newtheorem{proposition}{Proposition}%[section]
\newtheorem{remark}{Remark}%[section]

\makeatletter
\renewcommand{\maketag@@@}[1]{\hbox{\m@th\normalsize\normalfont#1}}%
\makeatother
\iccvfinalcopy % *** Uncomment this line for the final submission

\def\iccvPaperID{1997} % *** Enter the ICCV Paper ID here

\ificcvfinal\fi

\begin{document}

    %%%%%%%%% TITLE
    \title{DDFM: Denoising Diffusion Model for Multi-Modality Image Fusion}

    \author{Zixiang Zhao$^{1,2}$\quad
            Haowen Bai$^{1}$\quad
            Yuanzhi Zhu$^{2}$\quad
            Jiangshe Zhang$^{1}$\thanks{Corresponding author.}\quad
            Shuang Xu$^{3}$\\
            Yulun Zhang$^{2}$\quad
            Kai Zhang$^{2}$\quad
            Deyu Meng$^{1,5}$\quad
            Radu Timofte$^{2,4}$\quad
            Luc Van Gool$^{2}$\\[2mm]
		$^{1}$Xi’an Jiaotong University\quad
            $^{2}$Computer Vision Lab, ETH Z\"urich\\
            $^{3}$Northwestern Polytechnical University\quad
            $^{4}$University of W\"urzburg\\
            $^{5}$Macau University of Science and Technology\\
            {\tt\small zixiangzhao@stu.xjtu.edu.cn, jszhang@mail.xjtu.edu.cn}
}
\maketitle
\ificcvfinal\thispagestyle{empty}\fi

%%%%%%%%% ABSTRACT
\begin{abstract}
    Multi-modality image fusion aims to combine different modalities to produce fused images that retain the complementary features of each modality, such as functional highlights and texture details. To leverage strong generative priors and address challenges such as unstable training and lack of interpretability for GAN-based generative methods, we propose a novel fusion algorithm based on the denoising diffusion probabilistic model (DDPM). The fusion task is formulated as a conditional generation problem under the DDPM sampling framework, which is further divided into an unconditional generation subproblem and a maximum likelihood subproblem. The latter is modeled in a hierarchical Bayesian manner with latent variables and inferred by the expectation-maximization (EM) algorithm. By integrating the inference solution into the diffusion sampling iteration, our method can generate high-quality fused images with natural image generative priors and cross-modality information from source images. Note that all we required is an unconditional pre-trained generative model, and no fine-tuning is needed. Our extensive experiments indicate that our approach yields promising fusion results in infrared-visible image fusion and medical image fusion. The code is available at \url{https://github.com/Zhaozixiang1228/MMIF-DDFM}.

\end{abstract}

\begin{figure}[t]
    \centering
    \includegraphics[width=\linewidth]{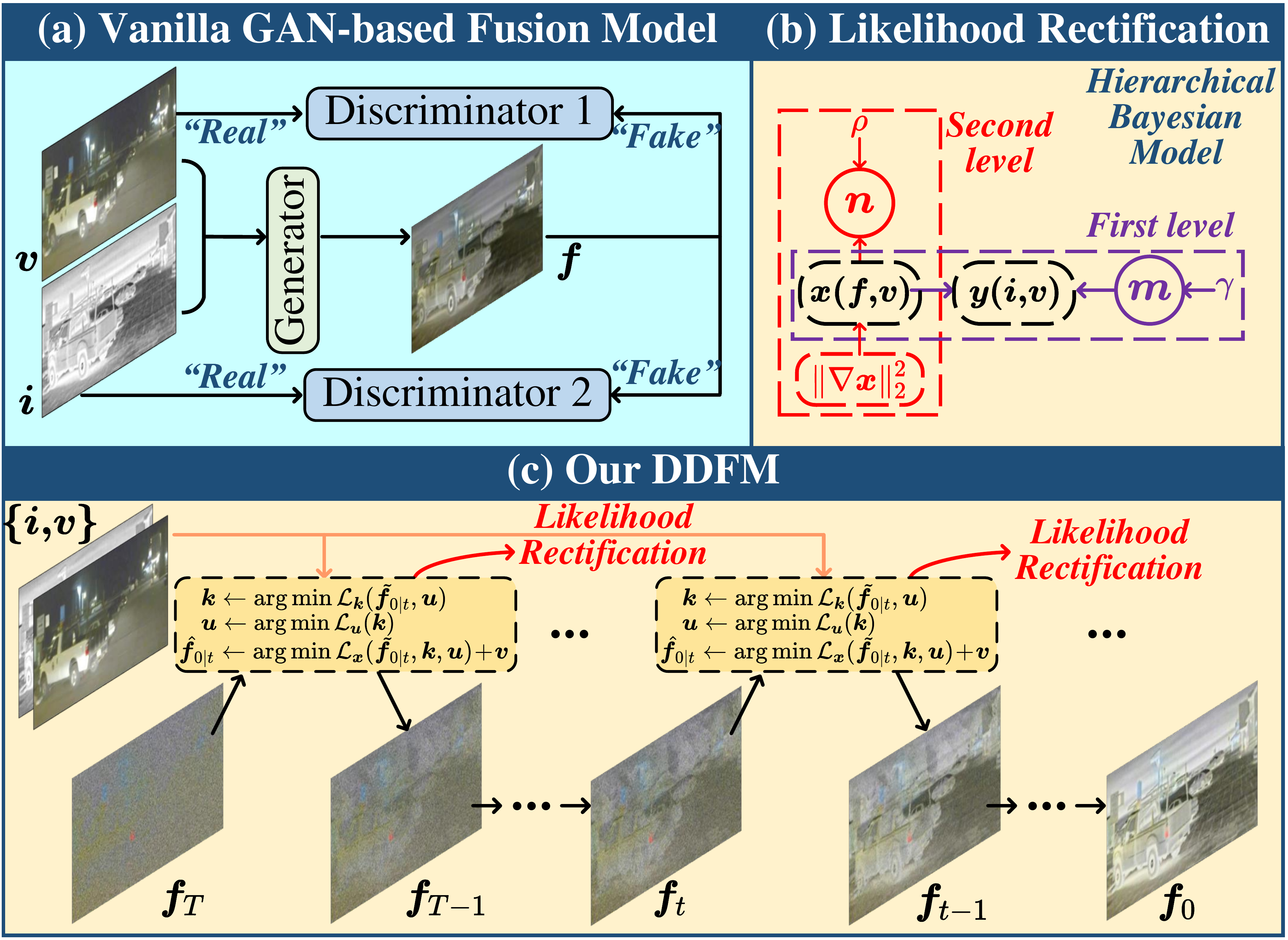}
    \caption{(a) Existing GAN-based fusion method workflow. (b) Graph of the hierarchical Bayesian model in likelihood rectification, linking the MMIF loss and our statistical inference model. (c) Our DDFM workflow: the unconditional diffusion sampling (UDS) module generates $\boldsymbol{f}_t$, while the likelihood rectification module, based on (b), rectifies UDS output with source image information.}
    \label{fig:introduction1}
\end{figure}
\begin{figure}[t]
    \vspace{-1em}
    \centering
    \includegraphics[width=\linewidth]{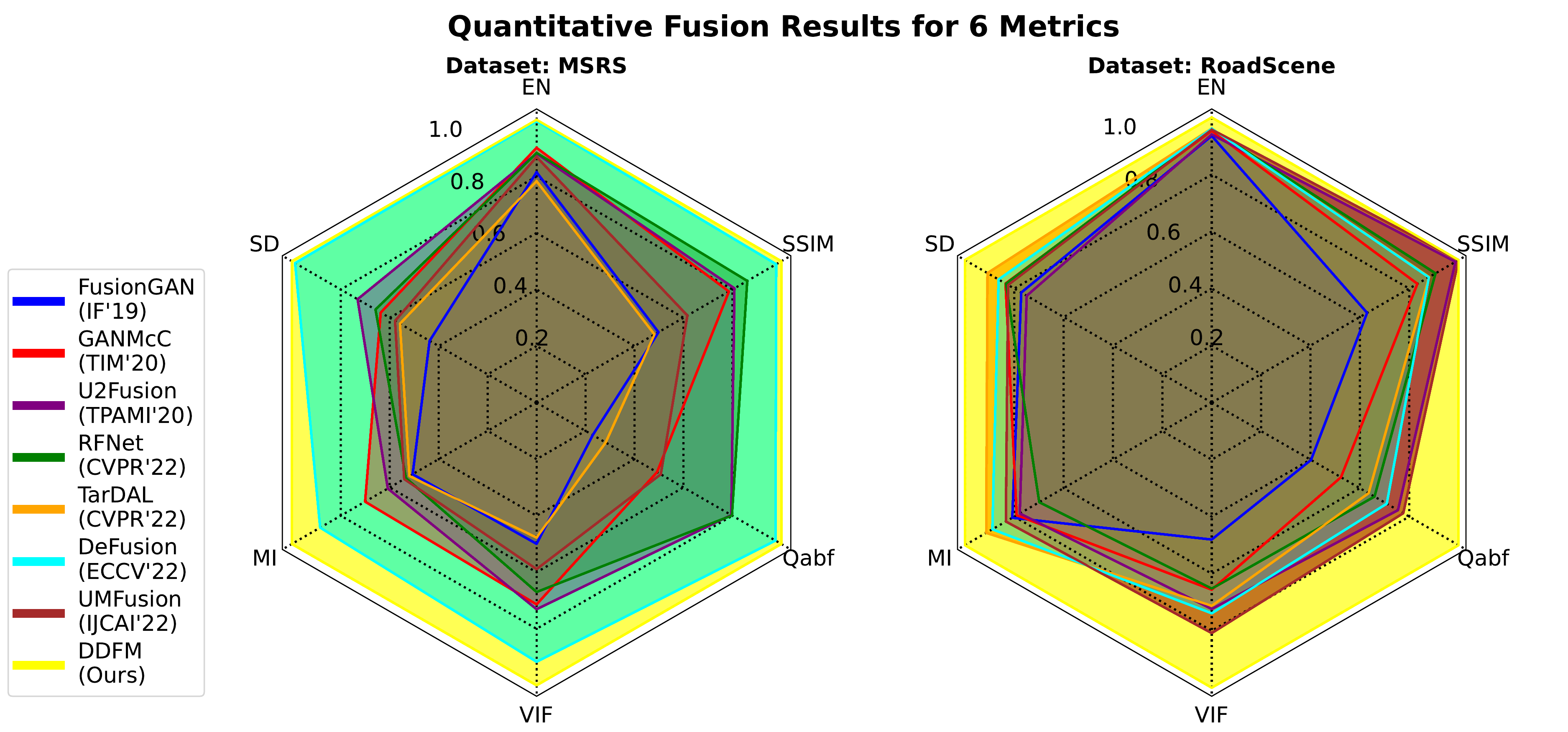}
    \caption{Visualization of results on MSRS~\cite{DBLP:journals/inffus/TangYZJM22} and RoadScene~\cite{xu2020aaai} in \cref{tab:Quantitative}. Hexagons formed by lines of different colors represent the values of different methods across six metrics. Our DDFM (marked in yellow) outperforms all other methods.}
    \label{fig:introduction}
    \vspace{-1em}
\end{figure}

\section{Introduction}\label{sec:1}
Image fusion integrates essential information from multiple source images to create high-quality fused images~\cite{meher2019a,zhao2023equivariant,liu2020bilevel,qin2022distribution}, encompassing various source image types like digital~\cite{DBLP:journals/tip/LiMYZ20,zhang2021deep,zhao2023spherical}, multi-modal~\cite{9151265,zhaoijcai2020}, and remote sensing~\cite{DBLP:conf/cvpr/Xu0ZSL021,DBLP:conf/icmcs/00010XSHL021}. This technology provides a clearer representation of objects and scenes, and has diverse applications such as saliency detection~\cite{DBLP:conf/cvpr/QinZHGDJ19,qin2023diverse,qin2023bibench}, object detection~\cite{he2023strategic,DBLP:journals/corr/abs-2004-10934,He2023Camouflaged,wang2021dual}, and semantic segmentation~\cite{DBLP:conf/mm/LiuLL021,he2023weaklysupervised,wang2022defensive}. Among the different subcategories of image fusion, \textit{Infrared-Visible image Fusion} (IVF) and \textit{Medical Image Fusion} (MIF) are particularly challenging in \textit{Multi-Modality Image Fusion} (MMIF) since they focus on modeling cross-modality features and preserving critical information from all sensors and modalities. Specifically, in IVF, fused images aim to retain both thermal radiation from infrared images and detailed texture information from visible images, thereby avoiding the limitations of visible images being sensitive to illumination conditions and infrared images being noisy and low-resolution. While MIF can assist in diagnosis and treatment by fusing multiple medical imaging modalities for precise detection of abnormality locations~\cite{DBLP:journals/inffus/JamesD14,he2023HQG}.

There have been numerous methods devised recently to address the challenges posed by MMIF~\cite{DBLP:conf/cvpr/LiuFHWLZL22,DBLP:journals/ijcv/ZhangM21,liu2023bi}, and generative models~\cite{DBLP:conf/nips/GoodfellowPMXWOCB14,mirza2014conditional} have been extensively utilized to model the distribution of fused images and achieve satisfactory fusion effects.
Among them, models based on Generative Adversarial Networks (GANs) \cite{ma2019fusiongan,DBLP:journals/tim/MaZSLX21,DBLP:journals/tip/MaXJMZ20,DBLP:conf/cvpr/LiuFHWLZL22} are dominant. The workflow of GAN-based models, illustrated in \cref{fig:introduction1}\red{a}, involves a generator that creates images containing information from source images, and a discriminator that determines whether the generated images are in a similar manifold to the source images.
Although GAN-based methods have the ability to generate high-quality fused images, they suffer from unstable training, lack of interpretability and mode collapse, which seriously affect the quality of the generated samples. Moreover, as a black-box model, it is difficult to comprehend the internal mechanisms and behaviors of GANs, making it challenging to achieve controllable generation.

Recently, \textit{Denoising Diffusion Probabilistic Models} (DDPM) \cite{DBLP:conf/nips/HoJA20} has garnered attention in the machine learning community, which generates high-quality images by modeling the diffusion process of restoring a noise-corrupted image towards a clean image. Based on the Langevin diffusion process, DDPM utilizes a series of reverse diffusion steps to generate promising synthetic samples~\cite{DBLP:conf/iclr/SongME21}.
Compared to GAN, DDPM does not require the discriminator network, thus mitigating common issues such as unstable training and mode collapse in GAN. Moreover, its generation process is interpretable, as it is based on denoising diffusion to generate images, enabling a better understanding of the image generation process~\cite{DBLP:conf/iclr/XiaoKV22}.

Therefore, we propose a \textbf{D}enoising \textbf{D}iffusion image \textbf{F}usion \textbf{M}odel (\textbf{DDFM}), as shown in \cref{fig:introduction1}\red{c}.
We formulate the conditional generation task as a DDPM-based posterior sampling model, which can be further decomposed into an unconditional generation diffusion problem and a maximum likelihood estimation problem. The former satisfies natural image prior while the latter is inferred to restrict the similarity with source images via {likelihood rectification}. Compared to discriminative approaches, modeling the natural image prior with DDPM enables better generation of details that are difficult to control by manually designed loss functions, resulting in visually perceptible images. As a generative method, DDFM achieves stable and controllable generation of fused images without discriminator, by applying likelihood rectification to the DDPM output.

Our contributions are organized in three aspects:
\begin{itemize}[itemsep=0em,topsep=0em,parsep=0pt]
    \item We introduce a DDPM-based posterior sampling model for MMIF, consisting of an unconditional generation module and a conditional likelihood rectification module. The sampling of fused images is achieved solely by a pre-trained DDPM without fine-tuning.
    \item In likelihood rectification, since obtaining the likelihood explicitly is not feasible, we formulate the optimization loss as a probability inference problem involving latent variables, which can be solved by the EM algorithm. Then the solution is integrated into the DDPM loop to complete conditional image generation.
    \item Extensive evaluation of IVF and MIF tasks shows that DDFM consistently delivers favorable fusion results, effectively preserving both the structure and detail information from the source images, while also satisfying visual fidelity requirements.
\end{itemize}

\section{Background}\label{sec:2}
\subsection{Score-based diffusion models}\label{sec:DDPM}
\bfsection{Score SDE formulation}
% \cite{DBLP:conf/icml/NicholD21,DBLP:conf/nips/HoJA20}
Diffusion models aim to generate samples by reversing a predefined forward process that converts a clean sample $\boldsymbol{x}_{0}$ to almost {Gaussian signal} $\boldsymbol{x}_{T}$ by gradually adding noise.
This forward process can be described by an It\^{o} Stochastic Differential Equation (SDE)~\cite{DBLP:conf/iclr/0011SKKEP21}:
\begin{equation}\label{eq:itoSDE}
    \mathrm{d}\boldsymbol{x} = -\frac{\beta(t)}{2} \boldsymbol{x}_{t}\mathrm{d}t + \sqrt{\beta(t)} \mathrm{d} \boldsymbol{w},
\end{equation}
where $\mathrm{d} \boldsymbol{w}$ is standard Wiener process and $\beta(t)$ is predefined noise schedule that favors the variance-preserving SDE~\cite{DBLP:conf/iclr/0011SKKEP21}.

This forward process can be reversed in time and still in the form of SDE \cite{anderson1982reverse}:
\begin{equation}\label{eq:revers_sde}
    \resizebox{.9\hsize}{!}{$
        \mathrm{d}\boldsymbol{x} = \left[-\frac{\beta(t)}{2}\boldsymbol{x}_{t} - \beta(t) \nabla_{\boldsymbol{x}_{t}} \log p_t(\boldsymbol{x_{t}})\right]\mathrm{d}t + \sqrt{\beta(t)} \mathrm{d} \overline{\boldsymbol{w}},$}
\end{equation}
where $\mathrm{d} \overline{\boldsymbol{w}}$ corresponds to the standard Wiener process running backward and the only unknown part $\nabla_{\boldsymbol{x}_{t}} \log p_t(\boldsymbol{x_{t}})$ can be modeled as the so-called \textit{score function} $\boldsymbol{s}_\theta(\boldsymbol{x}_{t},t)$ with denoising score matching methods, and this score function can be trained with the following objective \cite{hyvarinen2005estimation,song2019generative}:
\begin{equation}\label{eq:objective}
    \mathbb{E}_{t}\mathbb{E}_{\boldsymbol{x}_0}\mathbb{E}_{\boldsymbol{x}_t | \boldsymbol{x}_0 }
    \left[\|\boldsymbol{s}_\theta(\boldsymbol{x}_t, t) - \nabla_{\boldsymbol{x}_t}\log p_{0t}(\boldsymbol{x}_t | \boldsymbol{x}_0)\|_2^2 \right],
\end{equation}
where $t$ is uniformly sampled over $[0, T]$ and the data pair $(\boldsymbol{x}_0,\boldsymbol{x}_t) \sim p_0(\boldsymbol{x}) p_{0t}(\boldsymbol{x}_t |\boldsymbol{x}_0)$.

\bfsection{Sampling with diffusion models}
Specifically, an unconditional diffusion generation process starts with a random noise vector $\boldsymbol{x}_T \sim \mathcal{N}(\mathbf{0},\mathbf{I})$ and updates according to the discretization of \cref{eq:revers_sde}.
% as the input for the generation loop.
Alternatively, we can understand the sampling process in the DDIM fashion \cite{DBLP:conf/iclr/SongME21}, where the score function can also be considered to be a denoiser and predict the denoised $\tilde{\boldsymbol{x}}_{0|t}$ from any state $\boldsymbol{x}_t$ at iteration $t$:
\begin{equation}\label{eq:denoiser}
    \boldsymbol{\tilde{x}}_{0|t} = \frac{1}{\sqrt{\bar\alpha_t}}(\boldsymbol{x}_t + (1 - \bar\alpha_t)\boldsymbol{s}_\theta(\boldsymbol{x}_t,t)),
\end{equation}
and $\boldsymbol{\tilde{x}}_{0|t}$ denotes the estimation of $\boldsymbol{x}_0$ given $\boldsymbol{x}_t$. We use the same notation $\alpha_t=1-\beta_t$ and $\bar\alpha_t = \prod_{s=1}^t \alpha_s$ following Ho \etal \cite{DBLP:conf/nips/HoJA20}.
With this predicted $\boldsymbol{\tilde{x}}_{0|t}$ and the current state $\boldsymbol{x}_{t}$, $\boldsymbol{x}_{t-1}$ is updated from
\begin{equation}\label{eq:sampling_update}
    \boldsymbol{x}_{t-1} = \frac{\sqrt{\alpha_t}\left(1-\bar{\alpha}_{t-1}\right)}{1-\bar{\alpha}_t} \boldsymbol{x}_t+\frac{\sqrt{\bar{\alpha}_{t-1}} \beta_t}{1-\bar{\alpha}_t} \tilde{\boldsymbol{x}}_{0|t}+\tilde{\sigma}_t \boldsymbol{z},
\end{equation}
where $\boldsymbol{z}\sim \mathcal{N}(\mathbf{0},\mathbf{I})$ and $\tilde{\sigma}_t^2$ is the variance which is usually set to 0.
This sampled $\boldsymbol{x}_{t-1}$ is then fed into the next sampling iteration until the final image $\boldsymbol{x}_{0}$ is generated.
Further details about this sampling process can be found in the supplementary material or the original paper \cite{DBLP:conf/iclr/SongME21}.

\bfsection{Diffusion models applications}
Recently, diffusion models have been improved to generate images with better quality than previous generative models like GANs \cite{dhariwal2021diffusion,DBLP:conf/icml/NicholD21}.
% The state-of-the-art image-generating results from text-to-image diffusion models have drawn the attention of people
Moreover, diffusion models can be treated as a powerful generative prior and be applied to numerous conditional generation tasks.
One representative work with diffusion models is stable diffusion which can generate images according to given text prompts \cite{rombach2022high}.
Diffusion models are also applied to many low-level vision tasks. For instance, DDRM \cite{kawar2022denoising} performs diffusion sampling in the spectral space of degradation operator $\mathcal{A}$ to reconstruct the missing information in the observation $\boldsymbol{y}$.
DDNM \cite{wang2022ddnm} shares a similar idea with {DDRM} by refining the null-space of the operator $\mathcal{A}$ iteratively for image restoration tasks.
DPS \cite{DBLP:journals/corr/abs-2209-14687} endorses Laplacian approximation to calculate the gradient of log-likelihood for posterior sampling and it is capable of many noisy non-linear inverse problems.
In $\Pi$GDM \cite{song2023pseudoinverse}, the authors employ few approximations
to make the log-likelihood term tractable and hence make it able to solve inverse problems with even non-differentiable measurements.

\subsection{Multi-modal image fusion}
The deep learning-based multi-modality image fusion algorithms achieve effective feature extraction and information fusion through the powerful fitting ability of neural networks. Fusion algorithms are primarily divided into two branches: generative methods and discriminative methods.
For generative methods~\cite{ma2019fusiongan,ma2020infrared,DBLP:journals/tim/MaZSLX21}, particularly the GAN family, adversarial training~\cite{DBLP:conf/nips/GoodfellowPMXWOCB14,mao2017least,mirza2014conditional} is employed to generate fusion images following the same distribution as the source images.
For discriminative methods, auto encoder-based models~\cite{zhaoijcai2020,li2018densefuse,DBLP:journals/tim/LiWD20,DBLP:conf/mm/LiuLL021,DBLP:journals/corr/abs-2107-09011,DBLP:journals/inffus/LiWK21,DBLP:journals/ijcv/ZhangM21} use encoders and decoders to extract features and fuse them on a high-dimensional manifold.
Algorithm unfolding models~\cite{DBLP:journals/pami/0002D21,DBLP:journals/tip/GaoDXXD22,DBLP:journals/tcsv/ZhaoXZLZL22,DBLP:journals/corr/abs-2005-08448,DBLP:journals/corr/abs-2104-06977,li2023lrrnet} combine traditional optimization methods and neural networks, balancing efficiency and interpretability.
Unified models~\cite{xu2020aaai,DBLP:conf/aaai/ZhangXXGM20,9151265,DBLP:journals/inffus/ZhangLSYZZ20,DBLP:journals/tip/JungKJHS20} avoid the problem of lacking training data and ground truth for specific tasks.
Recently, CDDFuse~\cite{DBLP:journals/corr/abs-2211-14461} addresses cross-modality feature modeling and extracts modality-specific/shared features through a dual-branch Transformer-CNN architecture and correlation-driven loss, achieving promising fusion results in multiple fusion tasks.
On the other hand, fusion methods have been combined with pattern recognition tasks such as semantic segmentation~\cite{DBLP:journals/inffus/TangYM22} and object detection~\cite{DBLP:conf/cvpr/LiuFHWLZL22} to explore the interactions with downstream tasks. Specifically, TarDAL~\cite{DBLP:conf/cvpr/LiuFHWLZL22} demonstrates an obvious advantage in dealing with challenge scenarios with high efficiency. Self-supervised learning~\cite{Liang2022ECCV} is employed to train fusion networks without paired images. Moreover, the pre-processing registration module~\cite{DBLP:conf/cvpr/Xu0YLL22,huangreconet,DBLP:conf/ijcai/WangLFL22,xu2023murf} can enhance the robustness for unregistered input images. Benefiting from the multi-modality data, MSIS~\cite{jiang2022towards} achieves realizable and outstanding stitching results.
\subsection{Comparison with existing approaches}
The methods most relevant to our model are optimization-based methods and GAN-based generative methods.
Conventional optimization-based methods are often limited by manually designed loss functions, which may not be flexible enough to capture all relevant aspects and are sensitive to changes in the data distribution.
While incorporating natural image priors can provide extra knowledge that cannot be modeled by the generation loss function alone.
Then, in contrast to GAN-based generative methods, where unstable training and pattern collapse may occur, our DDFM achieves more stable and controllable fusion by rectifying the generation process towards source images and performing likelihood-based refinement in each iteration.
\section{Method}
In this section, we first present a novel approach for obtaining a fusion image by leveraging {DDPM posterior sampling}. Then, starting from the well-established loss function for image fusion, we derive a likelihood rectification approach for the unconditional DDPM sampling.
Finally, we propose the DDFM algorithm, which embeds the solution of the hierarchical Bayesian inference into the diffusion sampling. In addition, the rationality of the proposed algorithm will be demonstrated.
For brevity, we omit the derivations of some equations and refer interested readers to the \textit{supplementary material}.
It is worth noting that we use IVF as a case to illustrate our DDFM, and MIF can be carried out analogously to IVF.

\subsection{Fusing images via diffusion posterior sampling}
We first give the notation of the model formulation. Infrared, visible and fused images are denoted as $\boldsymbol{i}\!\in\!\mathbb{R}^{HW}$, $\boldsymbol{v}\!\in\!\mathbb{R}^{3HW}$ and $\boldsymbol{f}\!\in\!\mathbb{R}^{3HW}$, respectively.

We expect that the distribution of $\boldsymbol{f}$ given $\boldsymbol{i}$ and $\boldsymbol{v}$, \ie, $p\left(\boldsymbol{f}|\boldsymbol{i,\!v}\right)$, can be modeled, thus $\boldsymbol{f}$ can be obtained by sampling from the posterior distribution.
Inspired by \cref{eq:revers_sde}, we can express the reverse SDE of diffusion process as:
\begin{equation}
    \resizebox{.896\hsize}{!}{$d \boldsymbol{f}\!=\!\left[-\frac{\beta(t)}{2} \boldsymbol{f}\!-\!\beta(t) \nabla_{\boldsymbol{f}_t}\!\log p_t\left(\boldsymbol{f}_t|\boldsymbol{i,\!v}\right)\right]\!dt\!+\!\sqrt{\beta(t)} d \overline{\boldsymbol{w}}$},
\end{equation}
and the score function, \ie, $\nabla_{\boldsymbol{f}_t}\!\log p_t\!\left(\boldsymbol{f}_t|\boldsymbol{i,\!v}\right)$, can be calculated by:
\begin{equation}\label{eq:conditionscore}
    \resizebox{.9\hsize}{!}{$
        \begin{aligned}
            \nabla_{\boldsymbol{f}_t}\!\log p_t\!\left(\boldsymbol{f}_t|\boldsymbol{i,\!v}\right)\!&=\!\nabla_{\boldsymbol{f}_t}\!\log p_t\!\left(\boldsymbol{f}_t\right)\!+\!\nabla_{\boldsymbol{f}_t}\!\log p_t\!\left(\boldsymbol{i,\!v}|\boldsymbol{f}_t\right)\\
            &\approx \!\nabla_{\boldsymbol{f}_t}\!\log p_t\!\left(\boldsymbol{f}_t\right)\!+\!\nabla_{\boldsymbol{f}_t}\!\log p_t(\boldsymbol{i,\!v}|\tilde{\boldsymbol{f}}_{0|t})
        \end{aligned}
        $}
\end{equation}
where $\tilde{\boldsymbol{f}}_{0|t}$ is the estimation of $\boldsymbol{f}_0$ given $\boldsymbol{f}_t$ from the unconditional DDPM. The equality comes from {Bayes' theorem}, and the approximate equation is proved in \cite{DBLP:journals/corr/abs-2209-14687}.

In \cref{eq:conditionscore}, the first term represents the score function of unconditional diffusion sampling, which can be readily derived by the pre-trained DDPM. In the next section, we explicate the methodology for obtaining $\nabla_{\boldsymbol{f}_t}\!\log p_t(\boldsymbol{i,\!v}|\tilde{\boldsymbol{f}}_{0|t})$.

\subsection{Likelihood rectification for image fusion}\label{sec:BF}
Unlike the traditional image degradation inverse problem $\boldsymbol{y}=\mathcal{A}(\boldsymbol{x})+\boldsymbol{n}$ where $\boldsymbol{x}$ is the ground truth image, y is measurement and $\mathcal{A}(\cdot)$ is known, we can explicitly obtain its posterior distribution. However, it is not possible to explicitly express $p_t\!\left(\boldsymbol{i,\!v}|\boldsymbol{f}_t\right)$ or $p_t(\boldsymbol{i,\!v}|\tilde{\boldsymbol{f}}_{0|t})$ in image fusion. To address this, we start from the loss function and establish the relationship between the optimization loss function $\ell(\boldsymbol{i},\boldsymbol{v},\tilde{\boldsymbol{f}}_{0|t})$ and the likelihood $p_t(\boldsymbol{i},\!\boldsymbol{v}|\tilde{\boldsymbol{f}}_{0|t})$ of a probabilistic model.
For brevity, $\tilde{\boldsymbol{f}}_{0|t}$ is abbreviated as $\boldsymbol{f}$ in \cref{sec:321,sec:322}.
\subsubsection{Formulation of the likelihood model}\label{sec:321}
We first give a commonly-used loss function~\cite{DBLP:journals/inffus/LiWK21,DBLP:journals/ijcv/ZhangM21,ma2016infrared,DBLP:journals/corr/abs-2211-14461} for the image fusion task:
\begin{equation}\label{eq:lossfn0}
    \setlength{\abovedisplayskip}{5pt}
    \setlength{\belowdisplayskip}{5pt}
    \min_{\boldsymbol{f}}\|\boldsymbol{f}-\boldsymbol{i}\|_1+\phi\|\boldsymbol{f}-\boldsymbol{v}\|_1 .
\end{equation}
Then simple variable substitution $\boldsymbol{x}\!=\!\boldsymbol{f}\!-\!\boldsymbol{v}$ and $\boldsymbol{y}\!=\!\boldsymbol{i}\!-\!\boldsymbol{v}$ are implemented, and we get
\begin{equation}\label{eq:lossfn}
    \setlength{\abovedisplayskip}{5pt}
    \setlength{\belowdisplayskip}{5pt}
    \min_{\boldsymbol{x}}\|\boldsymbol{y}-\boldsymbol{x}\|_1+\phi\|\boldsymbol{x}\|_1.
\end{equation}
Since $\boldsymbol{y}$ is known and $\boldsymbol{x}$ is unknown, this $\ell_1$-norm optimization equation corresponds to the regression model:
$\boldsymbol{y}=\boldsymbol{kx} + \boldsymbol{\epsilon}$
with $\boldsymbol{k}$ fixed to $\boldsymbol{1}$. According to the relationship between regularization term and noise prior distribution, $\boldsymbol{\epsilon}$ should be a Laplacian noise and $\boldsymbol{x}$ is governed by the Laplacian distribution. Thus, in Bayesian fashion, we have:
\begin{equation}\label{eq:px-pyx}
    \small
    \begin{aligned}
        p(\boldsymbol{x})&\!=\!\mathcal{LAP}\left(\boldsymbol{x}; 0, \rho\right) \!=\!\prod_{i, j} \frac{1}{2 \rho} \exp \left(-\frac{\left|x_{i j}\right|}{\rho}\right),\\
        p(\boldsymbol{y}|\boldsymbol{x})&\!=\!\mathcal{LAP}\left(\boldsymbol{y};\boldsymbol{x},\gamma\right)\!=\!\prod_{i, j} \frac{1}{2 \gamma} \exp \left(-\frac{\left|y_{i j}\!-\!x_{i j}\right|}{\gamma}\right),
    \end{aligned}
\end{equation}
where $\mathcal{LAP}(\cdot)$ is the Laplacian distribution. $\rho$ and $\gamma$ are scale parameters of  $p(\boldsymbol{x})$ and $p(\boldsymbol{y}|\boldsymbol{x})$, respectively.

In order to prevent $\ell_1$-norm optimization in \cref{eq:lossfn} and inspired by \cite{ma2016infrared,ZHAO2020107734}, we give the \cref{proposition1}:
\begin{proposition}\label{proposition1}
    %    \vspace{-0.7em}
    For a random variable (RV) $\xi$ which obeys a Laplace distribution, it can be regarded as the coupling of a normally distributed RV and an exponentially distributed RV, which in formula:
    \begin{equation}
        \resizebox{.887\hsize}{!}{$
            \mathcal{LAP}(\xi ; \mu, \sqrt{b / 2})=\int_0^{\infty} \mathcal{N}(\xi ;\mu, a) \mathcal{EXP}(a ; b) d a.$}
    \end{equation}
\end{proposition}
\begin{remark}\label{remark1}
In \cref{proposition1}, we transform $\ell_1$-norm optimization into an $\ell_2$-norm optimization with latent variables, avoiding potential non-differentiable points in $\ell_1$-norm.
\end{remark}

Therefore, $p(\boldsymbol{x})$ and $p(\boldsymbol{y}|\boldsymbol{x})$ in \cref{eq:px-pyx} can be rewritten as the following hierarchical Bayesian framework:
\begin{equation}\label{eq:HBE}
    \left\{\begin{array}{l}
        y_{i j} | x_{i j}, m_{ij} \sim \mathcal{N}\left(y_{i j} ; x_{i j}, m_{ij}\right) \\
        m_{ij} \sim \mathcal{EXP}\left(m_{ij} ; \gamma\right) \\
        x_{i j} | n_{i j} \sim \mathcal{N}\left(x_{i j} ; 0, n_{i j}\right) \\
        n_{i j} \sim \mathcal{EXP}\left(n_{i j} ; \rho\right)
    \end{array}\right.
\end{equation}
where $i=1, \ldots, H$ and $j=1, \ldots, W$.
Through the above probabilistic analysis, the optimization problem in \cref{eq:lossfn} can be transformed into a {maximum likelihood} inference problem.

In addition, following~\cite{ma2016infrared,DBLP:journals/inffus/TangYM22}, the total variation penalty item $r(\boldsymbol{x}) = \|\nabla \boldsymbol{x}\|_2^2$ can be also added to make the fusion image $\boldsymbol{f}$ better preserve the texture information from $\boldsymbol{v}$, where $\nabla$ denotes the gradient operator. Ultimately, the log-likelihood function of the probabilistic inference issue is:
\begin{equation}\label{eq:loss}
    \begin{aligned}
        \ell(\boldsymbol{x}) & =\log p(\boldsymbol{x}, \boldsymbol{y})-r(\boldsymbol{x}) \\
        & =-\sum_{i, j}\left[\frac{\left(x_{i j}-y_{i j}\right)^2}{2 m_{ij}}+\frac{x_{i j}^2}{2 n_{i j}}\right]-\frac{\psi}{2}\|\nabla \boldsymbol{x}\|_2^2,
    \end{aligned}
\end{equation}
and probabilistic graph of this hierarchical Bayesian model is in \cref{fig:introduction1}\red{b}. Notably, in this way, we transform the optimization problem \cref{eq:lossfn0} into a maximum likelihood problem of a probability model \cref{eq:loss}.
Additionally, unlike traditional optimization methods that require manually specified tuning coefficients $\phi$ in \cref{eq:lossfn0},  $\phi$ in our model can be adaptively updated by inferring the latent variables, enabling the model to better fit different data distributions. The validity of this design has also been verified in ablation experiments in \cref{sec:ablation}.
We will then explore how to infer it in the next section.
\subsubsection{Inference the likelihood model via EM algorithm}\label{sec:322}
In order to solve the maximum log-likelihood problem in \cref{eq:loss}, which can be regarded as an optimization problem with latent variables, we use the \textit{Expectation Maximization} (EM) algorithm to obtain the optimal $\boldsymbol{x}$.
In \textit{E-step}, it calculates the expectation of log-likelihood function with respect to $p\left(\boldsymbol{a}, \boldsymbol{b} | \boldsymbol{x}^{(t)}, \boldsymbol{y}\right)$, \ie, the so-called $\mathcal{Q}$-function:%, is calculated by:
\begin{equation}\label{eq:Q-func}
    \mathcal{Q}\left(\boldsymbol{x} | \boldsymbol{x}^{(t)}\right)=\mathbb{E}_{\boldsymbol{a}, \boldsymbol{b} | \boldsymbol{x}^{(t)}, \boldsymbol{y}}[\ell(\boldsymbol{x})].
\end{equation}
Then in\textit{ M-step}, the optimal $\boldsymbol{x}$ is obtained by% maximize the $\mathcal{Q}$-function:
\begin{equation}
    \boldsymbol{x}^{(t+1)}=\arg \max_{\boldsymbol{x}} \mathcal{Q}\left(\boldsymbol{x} | \boldsymbol{x}^{(t)}\right).
\end{equation}
Next, we will show the implementation detail in each step.

\bfsection{E-step} \cref{proposition2} gives the calculation results for the conditional expectation of latent variables, and then gets the derivation of $\mathcal{Q}$-function.
\begin{proposition}\label{proposition2}
    %    \vspace{-0.7em}
    The conditional expectation of the latent variable ${1}/{m_{ij}}$ and ${1}/{n_{ij}}$ in \cref{eq:loss} are:
    \begin{equation}\label{eq:E-step}
        \small
        \begin{aligned} 		\mathbb{E}_{m_{ij}|x_{ij}^{(t)},y_{ij}}\!\left[\frac{1}{m_{ij}}\right]\!&=\!\sqrt{\frac{2(y_{ij}-x_{ij}^{(t)})^2}{\gamma}},\\
            \mathbb{E}_{n_{ij}|x_{ij}^{(t)}}\!\left[\frac{1}{n_{ij}}\right]\!&=\! \sqrt{\frac{2[x_{ij}^{(t)}]^2}{\rho}}.
        \end{aligned}
    \end{equation}
\end{proposition}
\begin{proof}
    For convenience, we set $\tilde{m}_{i j}\!\equiv\!1 / m_{ij}$ and $\tilde{n}_{i j}\!\equiv\!1 / n_{i j}$.
    From \cref{eq:HBE} we know that $m_{ij}\!\sim\!\mathcal{EXP}\left(m_{ij};\gamma\right)\!=\! \Gamma(m_{ij};1,\gamma)$. Thus, $\tilde{m}_{i j}\!\sim\!\mathcal{IG}\left(1,\gamma\right)$, where $\Gamma(\cdot,\cdot)$ and $\mathcal{IG}(\cdot,\cdot)$ are the gamma distribution and inverse gamma distribution, respectively.

    Then we can get the posterior of $\tilde{m}_{i j}$ by Bayes' theorem:
    \begin{equation}\label{eq:PostA}
        \small
        \begin{aligned}
            & \log p\left(\tilde{m}_{i j} | y_{i j}, x_{i j}\right)=\log p\left(y_{i j} | x_{i j}, m_{ij}\right)+\log p\left(\tilde{m}_{i j}\right) \\
            =& -\frac{3}{2} \log \tilde{m}_{i j}-\frac{\tilde{m}_{i j}\left(y_{i j}-x_{i j}\right)^2}{2}-\frac{1}{\gamma \tilde{m}_{i j}}+\text{ constant}.
        \end{aligned}
    \end{equation}
    Subsequently, we have
    \begin{equation}\label{eq:PostA2}
        \resizebox{.887\hsize}{!}{$
            p\left(\tilde{m}_{i j} | y_{i j}, x_{i j}\right)\!=\!\mathcal{I N}\!\left(\tilde{m}_{i j};\sqrt{2\left(y_{i j}\!-\!x_{i j}\right)^2 / \gamma}, 2 / \gamma\right),$}
    \end{equation}
    where $\mathcal{IN}(\cdot,\cdot)$ is the inverse Gaussian distribution. For the posterior of $\tilde{n}_{i j}$, it can be obtain similar to \cref{eq:PostA}:
    \begin{equation}\label{eq:PostB}
        \small
        \begin{aligned}
            & \log p\left(\tilde{n}_{i j} | x_{i j}\right)=\log p\left(x_{i j}  |  n_{i j}\right)+\log p\left(\tilde{n}_{i j}\right) \\
            & =-\frac{3}{2} \log \tilde{n}_{i j}-\frac{\tilde{n}_{i j} x_{i j}^2}{2}-\frac{1}{\rho \tilde{n}_{i j}}+\text{constant},
        \end{aligned}
    \end{equation}
    and therefore
    \begin{equation}\label{eq:PostB2}
        p\left(\tilde{n}_{i j} | x_{i j}\right)=\mathcal{I N}\left(\tilde{n}_{i j} ; \sqrt{2 x_{i j}^2 / \rho}, 2 / \rho\right).
    \end{equation}
    Finally, the conditional expectation of ${1}/{m_{ij}}$ and ${1}/{n_{ij}}$ are the mean parameters of the corresponding inverse Gaussian distribution \cref{eq:PostA2,eq:PostB2}, respectively.
\end{proof}
\begin{remark}\label{remark2}
The conditional expectation computed by \cref{proposition2} will be used to derive the $\mathcal{Q}$-function below.
\end{remark}
Afterwards, the $\mathcal{Q}$-function \cref{eq:Q-func} is derived as:
\begin{equation}
    \small
    \begin{aligned}
        \mathcal{Q} & =-\!\sum_{i, j}\!\left[\frac{\bar{m}_{ij}}{2}\!\left(x_{i j}\!-\!y_{i j}\right)^2\!+\!\frac{\bar{n}_{ij}}{2} x_{i j}^2\right]-\frac{\psi}{2}\|\nabla \boldsymbol{x}\|_2^2 \\
        & \propto -\left\|\boldsymbol{m} \odot(\boldsymbol{x}-\boldsymbol{y})\right\|_2^2-\left\|\boldsymbol{n} \odot \boldsymbol{x}\right\|_2^2-\psi\|\nabla \boldsymbol{x}\|_2^2,
    \end{aligned}
\end{equation}
where $\bar{m}_{ij}$ and $\bar{n}_{ij}$ represent $\mathbb{E}_{m_{ij}|x_{ij}^{(t)},y_{ij}}\!\left[{1}/{m_{ij}}\right]$ and $\mathbb{E}_{n_{ij}|x_{ij}^{(t)}}\!\left[{1}/{n_{ij}}\right]$ in \cref{eq:E-step}, respectively. $\odot$ is the element-wise multiplication. $\boldsymbol{m}$ and $\boldsymbol{n}$ are matrices with each element being $\sqrt{\bar{m}_{ij}}$ and $\sqrt{\bar{n}_{ij}}$, respectively.

\bfsection{M-step} Here, we need to minimize the negative $Q$-function with respect to $\boldsymbol{x}$. The half-quadratic splitting algorithm is employed to deal with this problem, i.e.,
\begin{equation}\label{eq:splitting}
    \begin{aligned}
        &\min_{\boldsymbol{x},\boldsymbol{u},\boldsymbol{k}} ||\boldsymbol{m}\odot(\boldsymbol{x}-\boldsymbol{y})||_2^2+||\boldsymbol{n}\odot \boldsymbol{x}||_2^2+ \psi||\boldsymbol{u}||_2^2, \\
        &{\rm s.t.}\quad \boldsymbol{u}=\nabla \boldsymbol{k}, \boldsymbol{k}=\boldsymbol{x}.
    \end{aligned}
\end{equation}
It can be further cast into the following unconstraint optimization problem,
\begin{equation}\label{eq:unconstraint}
    \begin{aligned}
        \min_{\boldsymbol{x},\boldsymbol{u},\boldsymbol{k}} & ||\boldsymbol{m}\odot(\boldsymbol{x}-\boldsymbol{y})||_2^2+||\boldsymbol{n}\odot \boldsymbol{x}||_2^2+ \psi||\boldsymbol{u}||_2^2  \\
        &+ \frac{\eta}{2}\left( ||\boldsymbol{u}-\nabla \boldsymbol{k}||_2^2+||\boldsymbol{k}-\boldsymbol{x}||_2^2 \right) .
    \end{aligned}
\end{equation}
The unknown variables $\boldsymbol{k},\boldsymbol{u},\boldsymbol{x}$ can be solved iteratively in the coordinate descent fashion.

\noindent\textbf{Update $\boldsymbol{k}$}: It is a deconvolution issue,
\begin{equation}
    \min_{\boldsymbol{k}} \mathcal{L}_{\boldsymbol{k}} = ||\boldsymbol{k}-\boldsymbol{x}||_2^2+||\boldsymbol{u}-\nabla \boldsymbol{k}||^2_2.
\end{equation}
It can be efficiently solved by the fast Fourier transform (fft) and inverse fft (ifft) operators, and the solution of $\boldsymbol{k}$ is
\begin{equation}\label{eq:K}
    \boldsymbol{k} = {\rm ifft} \left\lbrace \frac{{\rm fft}(\boldsymbol{x})+\overline{{\rm fft}(\nabla)}\odot{\rm fft}(\boldsymbol{u})}{1+\overline{{\rm fft}(\nabla)}\odot{\rm fft}(\nabla)} \right\rbrace,
\end{equation}
where $\overline{\cdot}$ is the complex conjugation.

\noindent\textbf{Update $\boldsymbol{u}$}: It is an $\ell_2$-norm penalized regression issue,
\begin{equation}
    \min_{\boldsymbol{u}} \mathcal{L}_{\boldsymbol{u}} = \psi||\boldsymbol{u}||_2^2+\frac{\eta}{2} ||\boldsymbol{u}-\nabla \boldsymbol{k}||_2^2.
\end{equation}
The solution of $\boldsymbol{u}$ is
\begin{equation}\label{eq:U}
    \boldsymbol{u} = \frac{\eta}{2\psi+\eta}\nabla \boldsymbol{k}.
\end{equation}

\noindent\textbf{Update $\boldsymbol{x}$}: It is a least squares issue,
\begin{equation}\label{eq:minX}
    \min_{\boldsymbol{x}} \mathcal{L}_{\boldsymbol{x}} = ||\boldsymbol{m}\odot(\boldsymbol{x}-\boldsymbol{y})||_2^2+||\boldsymbol{n}\odot \boldsymbol{x}||_2^2+\frac{\eta}{2}||\boldsymbol{k}-\boldsymbol{x}||_2^2.
\end{equation}
The solution of $\boldsymbol{x}$ is
\begin{equation} \label{eq:X}
    \boldsymbol{x} = (2\boldsymbol{m}^2\odot \boldsymbol{y}+\eta \boldsymbol{k}) \oslash (2\boldsymbol{m}^2+2\boldsymbol{n}^2+\eta),
\end{equation}
where $\oslash$ denotes the element-wise division, and final estimation of $\boldsymbol{f}$ is
\begin{equation} \label{eq:F}
    \boldsymbol{\hat{f}} = \boldsymbol{x} + \boldsymbol{v}.
\end{equation}

Additionally, hyper-parameter $\gamma$ and $\rho$ in \cref{eq:px-pyx} can be also updated after the sampling from $\boldsymbol{x}$ (\cref{eq:X}) by
\begin{equation}\label{eq:hp1}
    \gamma = \frac{1}{hw}\sum_{i,j} \mathbb{E}[m_{ij}],\ \rho = \frac{1}{hw}\sum_{i,j}\mathbb{E}[n_{ij}].
\end{equation}

\subsection{DDFM}
\bfsection{Overview}
In \cref{sec:BF}, we present a methodology for obtaining a hierarchical Bayesian model from existing loss function and perform the model inference via the EM algorithm.
In this section, we present our DDFM, where the inference solution and diffusion sampling are integrated within the same iterative framework for generating $\boldsymbol{f}_{0}$ given $\boldsymbol{i}$ and $\boldsymbol{v}$. The algorithm is illustrated in \cref{alg:moal,fig:introduction2}.

There are two modules in DDFM, the \textit{unconditional diffusion sampling}~(UDS) module and the \textit{likelihood rectification}, or say, {EM module}.
The UDS module is utilized to provide natural image priors, which improve the visual plausibility of the fused image. The EM module, on the other hand, is responsible for rectifying the output of UDS module via likelihood to preserve more information from the source images.

\bfsection{Unconditional diffusion sampling module}
In \cref{sec:DDPM}, we briefly introduce diffusion sampling. In \cref{alg:moal}, UDS (in \textit{grey}) is partitioned into two components, where the first part estimates $\tilde{\boldsymbol{f}}_{0|t}$ using $\boldsymbol{f}_t$, and the second part estimates $\boldsymbol{f}_{t-1}$ using both $\boldsymbol{f}_t$ and $\hat{\boldsymbol{f}}_{0|t}$. From the perspective of score-based DDPM in \cref{eq:conditionscore}, a pre-trained DDPM can directly output the current $\nabla_{\boldsymbol{f}_t}\!\log p_t\!\left(\boldsymbol{f}_t\right)$, while $\nabla_{\boldsymbol{f}_t}\!\log p_t(\boldsymbol{i,\!v}|\tilde{\boldsymbol{f}}_{0|t})$ can be obtain by the EM module.

\bfsection{EM module}
The role of the EM module is to update $\tilde{\boldsymbol{f}}_{0|t}\!\Rightarrow\!\hat{\boldsymbol{f}}_{0|t}$.
In \cref{alg:moal,fig:introduction2}, the EM algorithm (in \textit{blue} and \textit{yellow}) is inserted in UDS (in \textit{grey}). The preliminary estimate $\tilde{\boldsymbol{f}}_{0|t}$ produced by DDPM sampling (\textit{line 5}) is utilized as the initial input for the EM algorithm to obtain $\hat{\boldsymbol{f}}_{0|t}$ (\textit{line 6-13}), which is an estimation of the fused image subjected to likelihood rectification.
In other words, EM module rectify $\tilde{\boldsymbol{f}}_{0|t}$ to $\hat{\boldsymbol{f}}_{0|t}$ to meet the likelihood.

\begin{algorithm}[t]
    \caption{DDFM}
    \label{alg:moal}
    \begin{algorithmic}[1]
        \REQUIRE ~~\\ % Input
        Infrared image $\boldsymbol{i}$, Visible image $\boldsymbol{v}$, $T$, $\left\{\tilde{\sigma}_t\right\}_{t=1}^T$
        \ENSURE ~~\\ % Output
        Fused image $\boldsymbol{f}_0$.
        \STATE $\boldsymbol{f}_T \sim \mathcal{N}(\mathbf{0}, \mathbf{I})$
        \FOR {$t=T-1$ \text{\textbf{to}} 0}
        \vspace{-0.5em}
        \definecolor{shadecolor}{rgb}{0.92,0.92,0.92}
        \begin{shaded}
            \vspace{-0.7em}
            \STATE \% \textit{DDPM Part 1: Obtain} $\tilde{\boldsymbol{f}}_{0|t}$
            \STATE $\hat{\boldsymbol{s}} \leftarrow \boldsymbol{s}_\theta\left(\boldsymbol{f}_t, t\right)$
            \STATE $\tilde{\boldsymbol{f}}_{0|t} \leftarrow \frac{1}{\sqrt{\bar{\alpha}_t}}\left(\boldsymbol{f}_t+\left(1-\bar{\alpha}_t\right) \hat{\boldsymbol{s}}\right)$
        \end{shaded}
        \vspace{-1.4em}
        \definecolor{shadecolor}{rgb}{0.54, 0.81, 0.94}
        \begin{shaded}
            \vspace{-0.7em}
            \STATE \% \textit{E-step: Update latent variables}
            \STATE $\tilde{\boldsymbol{x}}_{0}=\tilde{\boldsymbol{f}}_{0|t} - \boldsymbol{v}$, $\boldsymbol{y}=\boldsymbol{i}-\boldsymbol{v}$
            \STATE Evaluate expectations by \cref{eq:E-step}.
            \STATE Update hyper-parameters $\gamma,\rho$ by \cref{eq:hp1}.
        \end{shaded}
        \vspace{-1.4em}
        \definecolor{shadecolor}{rgb}{0.9,0.9,0.5}
        \begin{shaded}
            \vspace{-0.8em}
            \STATE \% \textit{M-step: Obtain $\hat{\boldsymbol{f}}_{0|t}$ via Likelihood Rectification}
            \STATE $\boldsymbol{k} \leftarrow \arg\min_{\boldsymbol{k}} \mathcal{L}_{\boldsymbol{k}}\left(\tilde{\boldsymbol{x}}_{0},\boldsymbol{u}\right)$ \quad (\cref{eq:K})
            \STATE $\boldsymbol{u} \leftarrow \arg\min_{\boldsymbol{u}} \mathcal{L}_{\boldsymbol{u}}\left(\nabla \boldsymbol{k}\right)$  \ \  \quad (\cref{eq:U})
            \STATE $\hat{\boldsymbol{f}}_{0|t} \leftarrow \arg\min_{\boldsymbol{x}} \mathcal{L}_{\boldsymbol{x}}\left(\boldsymbol{k},\boldsymbol{u},\tilde{\boldsymbol{x}}_{0}\right)+\boldsymbol{v}$ \  (Eq.~(\ref{eq:X})\&(\ref{eq:F}))%(\cref{eq:X,eq:F})
        \end{shaded}
        \vspace{-1.4em}
        \definecolor{shadecolor}{rgb}{0.92,0.92,0.92}
        \begin{shaded}
            \vspace{-0.7em}
            \STATE \% \textit{DDPM Part 2: Estimate} $\boldsymbol{f}_{t-1}$
            \STATE $\boldsymbol{z} \sim \mathcal{N}(\mathbf{0}, \mathbf{I})$
            \STATE $\boldsymbol{f}_{t-1} \leftarrow \frac{\sqrt{\alpha_t}\left(1-\bar{\alpha}_{t-1}\right)}{1-\bar{\alpha}_t} \boldsymbol{f}_t+\frac{\sqrt{\bar{\alpha}_{t-1}} \beta_t}{1-\bar{\alpha}_t} \hat{\boldsymbol{f}}_{0|t}+\tilde{\sigma}_t \boldsymbol{z}$
            \vspace{-0.9em}
        \end{shaded}
        \vspace{-0.3em}
        \ENDFOR
    \end{algorithmic}
\end{algorithm}

\begin{figure}[t]
    \vspace{-0.5em}
    \centering
    \includegraphics[width=\linewidth]{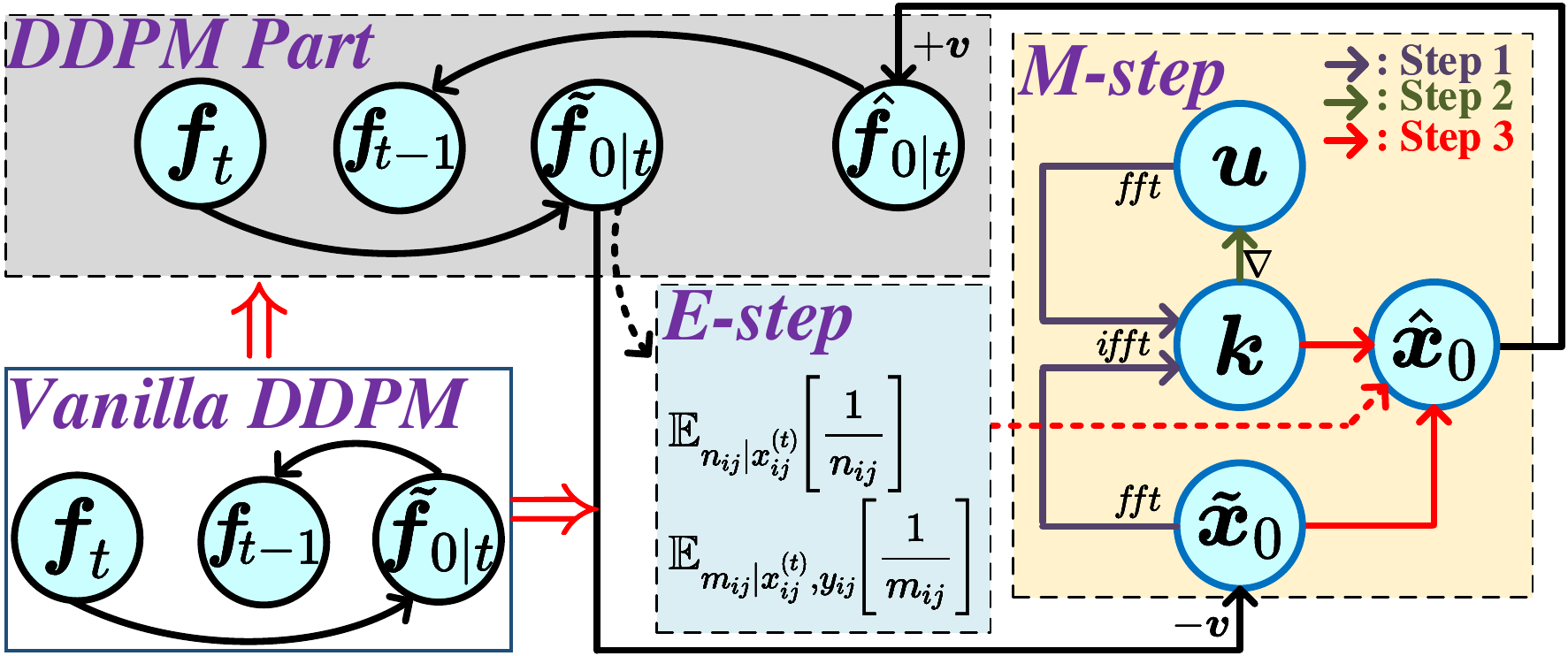}
    \caption{Computational graph of our DDFM in one iteration. Different from the vanilla DDPM, likelihood rectification is completed via the EM algorithm, \ie, the update from $\tilde{\boldsymbol{f}}_{0|t} \Rightarrow \hat{\boldsymbol{f}}_{0|t}$.}
    \label{fig:introduction2}
    %\vspace{-0.2em}
\end{figure}

\begin{figure*}[t]
    \centering
    \includegraphics[width=\linewidth]{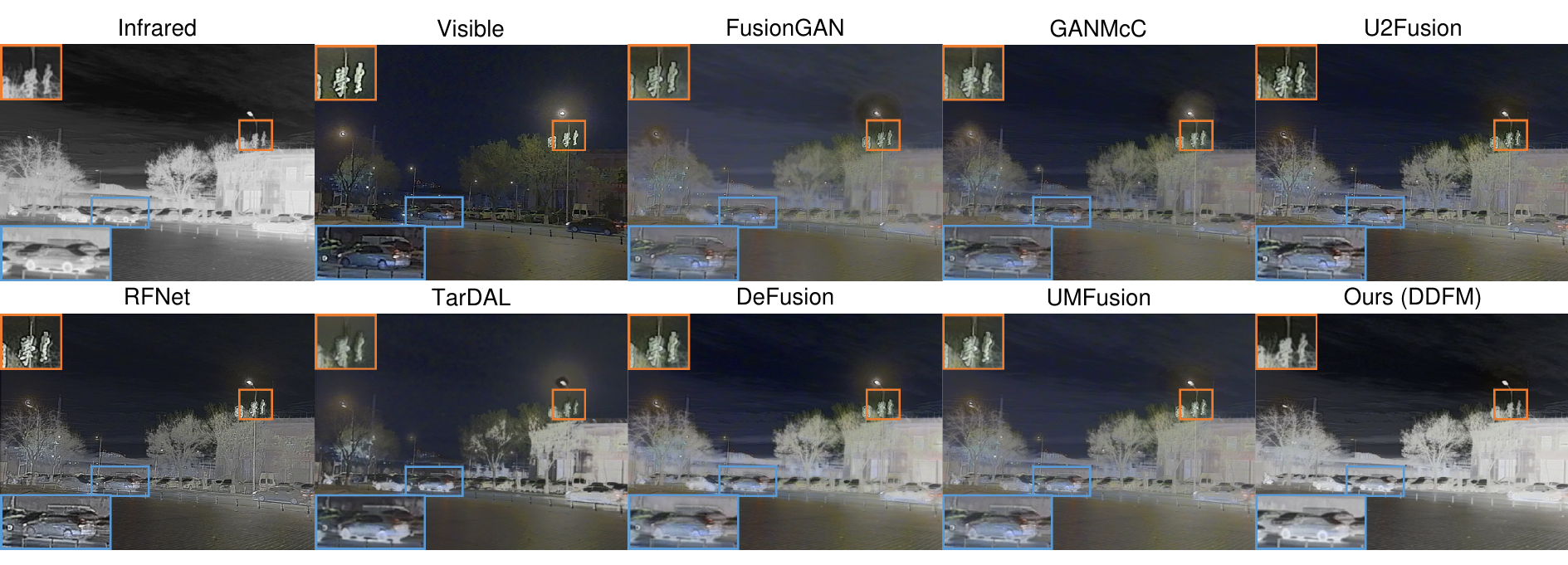}
    %\vspace{-0.8em}
    \caption{Visual comparison of ``01462'' from M$^3$FD IVF dataset.}
    \label{fig:IVF2}
\end{figure*}

\begin{figure*}[t]
    %\vspace{-1.2em}
    \centering
    \includegraphics[width=\linewidth]{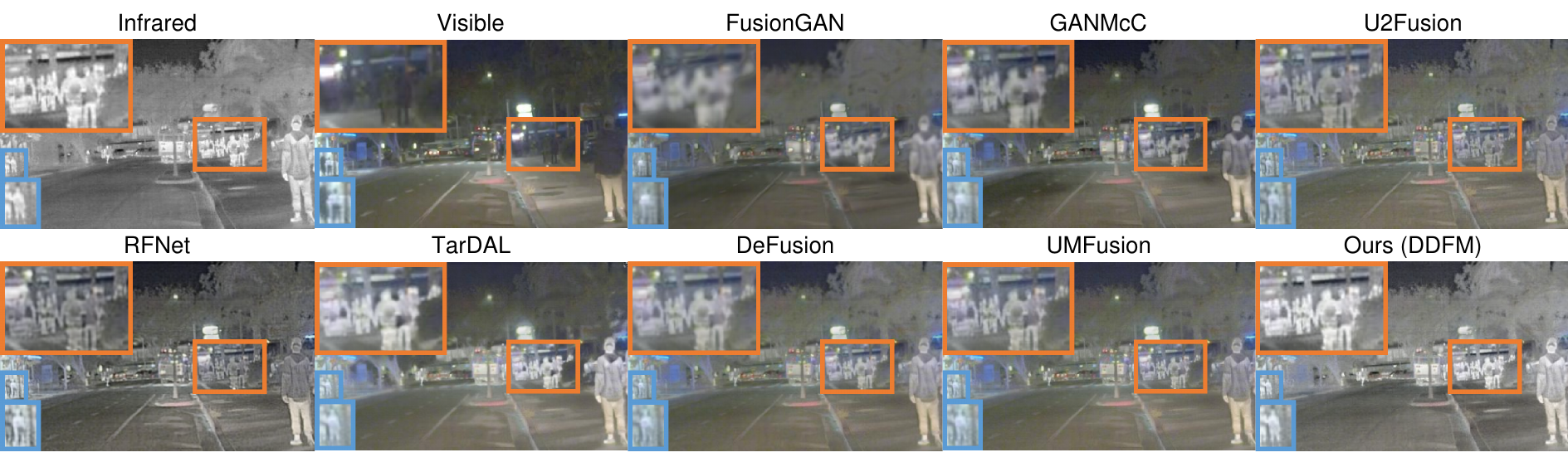}
    %\vspace{-0.8em}
    \caption{Visual comparison of ``FLIR\_08248'' from RoadScene IVF dataset.}
    \label{fig:IVF3}
    %\vspace{-1.3em}
\end{figure*}

\begin{table*}[t]
    \centering
    \caption{The quantitative results of IVF task, with the best and second-best values in \textbf{boldface} and \underline{underline}, respectively.}
    \label{tab:Quantitative}%
    \resizebox{\linewidth}{!}{
        \begin{tabular}{cccccccccccccc}
            \toprule
            \multicolumn{7}{c}{\textbf{Dataset: MSRS Fusion Dataset}~\cite{DBLP:journals/inffus/TangYZJM22}}                              &                              \multicolumn{7}{c}{\textbf{Dataset: M$^3$FD Fusion Dataset}~\cite{DBLP:conf/cvpr/LiuFHWLZL22}}                               \\
            &  EN $\uparrow$   &   SD $\uparrow$   &  MI $\uparrow$   &  VIF $\uparrow$  & Qabf $\uparrow$  & SSIM $\uparrow$  &                                        &  EN $\uparrow$   &   SD $\uparrow$   &  MI $\uparrow$   &  VIF $\uparrow$  & Qabf $\uparrow$  & SSIM $\uparrow$  \\ \midrule
            FGAN~\cite{ma2019fusiongan}       &       5.60       &       17.81       &       1.29       &       0.40       &       0.13       &       0.47       &      FGAN~\cite{ma2019fusiongan}       &       6.51       &       28.14       &       2.07       &       0.44       &       0.30       &       0.75       \\
            GMcC~\cite{DBLP:journals/tim/MaZSLX21} &       6.20       &       25.95       &       1.79       &       0.57       &       0.28       &       0.74       & GMcC~\cite{DBLP:journals/tim/MaZSLX21} &       6.68       &       32.23       &       2.01       &       0.58       &       0.36       &       0.93       \\
            U2F~\cite{9151265}           &       6.06       &       29.80       &       1.55       &       0.59       &       0.46       &       0.76       &           U2F~\cite{9151265}           & \underline{6.84} &       34.05       &       1.95       & \underline{0.73} & \underline{0.49} &  \textbf{0.98}   \\
            RFN~\cite{DBLP:conf/cvpr/Xu0YLL22}   &       6.07       &       26.82       &       1.36       &       0.54       &       0.46       &       0.81       &   RFN~\cite{DBLP:conf/cvpr/Xu0YLL22}   &       6.67       &       31.04       &       1.71       &       0.67       &       0.44       &       0.91       \\
            TarD~\cite{DBLP:conf/cvpr/LiuFHWLZL22} &       5.39       &       22.74       &       1.32       &       0.38       &       0.16       &       0.45       & TarD~\cite{DBLP:conf/cvpr/LiuFHWLZL22} &       6.67       & \underline{38.83} & \underline{2.38} &       0.54       &       0.29       &       0.87       \\
            DeF~\cite{Liang2022ECCV}        & \underline{6.85} & \underline{40.20} & \underline{2.25} & \underline{0.74} & \underline{0.56} & \underline{0.92} &        DeF~\cite{Liang2022ECCV}        &       6.79       &       36.39       &       2.32       &       0.65       &       0.44       &       0.94       \\
            UMF~\cite{DBLP:conf/ijcai/WangLFL22}  &       5.98       &       23.56       &       1.38       &       0.47       &       0.29       &       0.58       &  UMF~\cite{DBLP:conf/ijcai/WangLFL22}  &       6.73       &       32.46       &       2.23       &       0.66       &       0.40       & \underline{0.97} \\
            Ours                  &  \textbf{6.88}   &  \textbf{40.75}   &  \textbf{2.35}   &  \textbf{0.81}   &  \textbf{0.58}   &  \textbf{0.94}   &                  Ours                  &  \textbf{6.86}   &  \textbf{38.95}   &  \textbf{2.52}   &  \textbf{0.80}   &  \textbf{0.49}   &       0.95       \\ \midrule
            \multicolumn{7}{c}{\textbf{Dataset: RoadScene Fusion Dataset}~\cite{xu2020aaai}}                                      &                                            \multicolumn{7}{c}{\textbf{Dataset: TNO Fusion Dataset}~\cite{TNO}}                                            \\
            &  EN $\uparrow$   &   SD $\uparrow$   &  MI $\uparrow$   &  VIF $\uparrow$  & Qabf $\uparrow$  & SSIM $\uparrow$  &                                        &  EN $\uparrow$   &   SD $\uparrow$   &  MI $\uparrow$   &  VIF $\uparrow$  & Qabf $\uparrow$  & SSIM $\uparrow$  \\ \midrule
            FGAN~\cite{ma2019fusiongan}       &       7.12       &       40.13       &       1.90       &       0.36       &       0.26       &       0.61       &      FGAN~\cite{ma2019fusiongan}       &       6.74       &       34.41       &       1.78       &       0.42       &       0.25       &       0.66       \\
            GMcC~\cite{DBLP:journals/tim/MaZSLX21} &       7.26       &       43.44       &       1.86       &       0.49       &       0.34       &       0.81       & GMcC~\cite{DBLP:journals/tim/MaZSLX21} &       6.86       &       35.51       &       1.64       &       0.53       &       0.28       &       0.83       \\
            U2F~\cite{9151265}           &       7.16       &       38.97       &       1.83       &       0.54       &       0.49       &       0.96       &           U2F~\cite{9151265}           &       7.02       &       38.52       &       1.41       &       0.63       & \underline{0.43} &       0.93       \\
            RFN~\cite{DBLP:conf/cvpr/Xu0YLL22}   &       7.30       &       43.37       &       1.64       &       0.49       &       0.43       &       0.88       &   RFN~\cite{DBLP:conf/cvpr/Xu0YLL22}   &       6.93       &       34.95       &       1.21       &       0.55       &       0.37       &       0.87       \\
            TarD~\cite{DBLP:conf/cvpr/LiuFHWLZL22} & \underline{7.31} & \underline{47.24} & \underline{2.15} &       0.53       &       0.41       &       0.86       & TarD~\cite{DBLP:conf/cvpr/LiuFHWLZL22} &       7.02       & \underline{49.89} &       1.89       &       0.54       &       0.28       &       0.83       \\
            DeF~\cite{Liang2022ECCV}        &       7.31       &       44.91       &       2.09       &       0.55       &       0.46       &       0.86       &        DeF~\cite{Liang2022ECCV}        & \underline{7.06} &       40.70       & \underline{2.04} &       0.64       &       0.43       &       0.92       \\
            UMF~\cite{DBLP:conf/ijcai/WangLFL22}  &       7.29       &       42.91       &       1.96       & \underline{0.61} & \underline{0.50} & \underline{0.98} &  UMF~\cite{DBLP:conf/ijcai/WangLFL22}  &       6.83       &       36.56       &       1.66       & \underline{0.65} &       0.42       & \underline{0.94} \\
            Ours                  &  \textbf{7.41}   &  \textbf{52.61}   &  \textbf{2.35}   &  \textbf{0.75}   &  \textbf{0.65}   &  \textbf{0.98}   &                  Ours                  &  \textbf{7.06}   &  \textbf{51.42}   &  \textbf{2.21}   &  \textbf{0.81}   &  \textbf{0.49}   &  \textbf{0.95}   \\ \bottomrule
    \end{tabular}}
\end{table*}%

\begin{table}[t]
    \centering
    \caption{Ablation experiment results. \textbf{Bold} indicates the best value.}
    \label{tab:ablation}%
    \resizebox{\linewidth}{!}{
        \begin{tabular}{cccccccccc}
            \toprule
            \multirow{2}{*}{}            &      \multicolumn{3}{c}{Configurations}      & \multirow{2}{*}{EN} & \multirow{2}{*}{SD} & \multirow{2}{*}{MI} & \multirow{2}{*}{VIF} & \multirow{2}{*}{Qabf} & \multirow{2}{*}{SSIM} \\
            \cmidrule(lr){2-4}           &  DDPM   & $r(\boldsymbol{x})$ &    $\phi$    &                     &                     &                      &                     &                     &                     \\ \midrule
            \uppercase\expandafter{\romannumeral1} &         &       $\surd$       & $\backslash$ &        7.19         &        41.82        &         2.11         &        0.60         &        0.42         &        0.92         \\
            \uppercase\expandafter{\romannumeral2} & $\surd$ &                     & $\backslash$ &        7.33         &        44.12        &         2.29         &        0.69         &        0.52         &        0.93         \\
            \uppercase\expandafter{\romannumeral3} & $\surd$ &       $\surd$       &     0.1      &        7.25         &        43.16        &         2.26         &        0.66         &        0.49         &        0.90         \\
            \uppercase\expandafter{\romannumeral4} & $\surd$ &       $\surd$       &      1       &        7.26         &        42.37        &         2.24         &        0.66         &        0.47         &        0.91         \\ \midrule
            Ours                  & $\surd$ &       $\surd$       & $\backslash$ &    \textbf{7.41}    &   \textbf{52.61}    &    \textbf{2.35}     &    \textbf{0.75}    &    \textbf{0.65}    &    \textbf{0.98}    \\ \bottomrule
    \end{tabular}}
\end{table}%

\subsection{Why does one-step EM work?}
The main difference between our DDFM and conventional EM algorithm lies in that the traditional method requires numerous iterations to obtain the optimal $\boldsymbol{x}$, \ie, the operation from line 6-13 in \cref{alg:moal} needs to be looped many times. However, our DDFM only requires one step of the EM algorithm iteration, which is embedded into the DDPM framework to accomplish sampling. In the following, we give \cref{proposition3} to demonstrate its rationality.
\begin{proposition}\label{proposition3}
    %    \vspace{-0.7em}
    One-step unconditional diffusion sampling combined with one-step EM iteration is equivalent to one-step conditional diffusion sampling.
\end{proposition}
\begin{proof}
    The estimation of $\hat{\boldsymbol{f}}_{0|t}$ in conditional diffusion sampling, refer to \cref{eq:denoiser}, could be expressed as:
    \begin{subequations}
    % \vspace{-0.1em}
    %     \setlength{\abovedisplayskip}{0pt}
    %     \setlength{\belowdisplayskip}{0pt}
        \small
        \begin{align}
            & \hat{f}_{0|t}\left(\boldsymbol{f}_t,\boldsymbol{i},\boldsymbol{v}\right)=\frac{1}{\sqrt{\bar{\alpha}_t}}\left[\boldsymbol{f}_t+\left(1-\bar{\alpha}_t\right) \boldsymbol{s}_\theta\left(\boldsymbol{f}_t, \boldsymbol{i},\boldsymbol{v}\right)\right]                        \label{equ:proofsub1} \\
            =       & \frac{1}{\sqrt{\bar{\alpha}_t}}\left\{\boldsymbol{f}_t\!+\!\left(1\!-\!\bar{\alpha}_t\right)\!\left[\boldsymbol{s}_\theta\!\left(\boldsymbol{f}_t\right)\!+\!\nabla_{\boldsymbol{f}_t}\!\log p_t\!\left(\boldsymbol{i,\!v}|\boldsymbol{f}_t\right)\right]\right\} \label{equ:proofsub2} \\
            \approx & \tilde{f}_{0|t}\left(\boldsymbol{f}_t\right)+\frac{1-\bar{\alpha}_t}{\sqrt{\bar{\alpha}_t}} \nabla_{\boldsymbol{f_t}}\!\log p_t\left(\boldsymbol{i,v}|\tilde{\boldsymbol{f}}_{0|t}\right)                                                                                      \label{equ:proofsub3} \\
            =       & \tilde{f}_{0|t}\left(\boldsymbol{f}_t\right)-\zeta_t \nabla_{\tilde{\boldsymbol{x}}_{0}}\mathcal{L}_{\boldsymbol{x}}\left(\boldsymbol{i},\boldsymbol{v},\tilde{\boldsymbol{x}}_{0}\right). \label{equ:proofsub4}
        \end{align}
    \end{subequations}
    \cref{equ:proofsub1,equ:proofsub2,equ:proofsub3} are respectively based on the definition of Score-based DDPM, Bayes' theorem, and proof in \cite{DBLP:journals/corr/abs-2209-14687}.
    For \cref{equ:proofsub4}, although optimization \cref{eq:minX} has a closed-form solution (\cref{eq:X}), it can also be solved by gradient descent:
    \begin{equation}
        \resizebox{.88\hsize}{!}{$
            \hat{\boldsymbol{x}}_0 = \tilde{\boldsymbol{x}}_{0} + \nabla_{\tilde{\boldsymbol{x}}_{0}}\mathcal{L}_{\boldsymbol{x}}\left(\boldsymbol{k},\boldsymbol{u},\tilde{\boldsymbol{x}}_{0}\right)
            =\tilde{\boldsymbol{x}}_{0} + \nabla_{\tilde{\boldsymbol{x}}_{0}}\mathcal{L}_{\boldsymbol{x}}\left(\boldsymbol{i},\boldsymbol{v},\tilde{\boldsymbol{x}}_{0}\right)$}
    \end{equation}
    where the second equation holds true because as the input for updating $\hat{\boldsymbol{x}}_0$ (\cref{eq:X}), $\boldsymbol{k}$ and $\boldsymbol{u}$ are functions of $\boldsymbol{i}$ and $\boldsymbol{v}$. $\zeta_t$ in \cref{equ:proofsub4} can be regraded as the update step size.

    Hence, conditional sampling $\hat{f}_{0|t}\!\left(\boldsymbol{f}_t,\boldsymbol{i},\boldsymbol{v}\right)$ can be split as an unconditional diffusion sampling $\tilde{f}_{0|t}\!\left(\boldsymbol{f}_t\right)$ and one-step EM iteration  $\nabla_{\tilde{\boldsymbol{x}}_{0}}\mathcal{L}_{\boldsymbol{x}}\!\left(\boldsymbol{i},\boldsymbol{v},\tilde{\boldsymbol{x}}_{0}\right)$, corresponding to UDS module (part 1) and EM module, respectively.
\end{proof}
\begin{remark}\label{remark3}
\cref{proposition3} demonstrates the theoretical explanation for the rationality of inserting the EM module into the UDS module and explains why the EM module only involves one iteration of the Bayesian inference algorithm.
\end{remark}

\section{Infrared and visible image fusion}\label{sec:experiment}
In this section, we elaborate on numerous experiments for IVF task to demonstrate the superiority of our method. More
related experiments are placed in \textit{supplementary material}.
\subsection{Setup}
\bfsection{Datasets and pre-trained model}
Following the protocol in \cite{DBLP:conf/cvpr/LiuFHWLZL22,Liang2022ECCV}, IVF experiments are conducted on the four test datasets, \ie, TNO~\cite{TNO}, RoadScene~\cite{xu2020aaai}, MSRS~\cite{DBLP:journals/inffus/TangYZJM22}, and M$^3$FD~\cite{DBLP:conf/cvpr/LiuFHWLZL22}. Note that there is no training dataset due to that we do not need any fine-tuning for specific tasks but directly use the pre-trained DDPM model. We choose the pre-trained model proposed by \cite{dhariwal2021diffusion}, which is trained on ImageNet~\cite{DBLP:journals/ijcv/RussakovskyDSKS15}.

\bfsection{Metrics}
We employ six metrics including entropy (EN), standard deviation (SD), mutual information (MI), visual information fidelity (VIF), $Q^{AB/F}$, and structural similarity index measure (SSIM) in the quantitative experiments to comprehensively evaluate the fused effect. The detail of metrics is in \cite{ma2019infrared}.

\bfsection{Implement details}
We use a machine with one NVIDIA GeForce RTX 3090 GPU for fusion image generation. All input images are normalized to $[-1, 1]$. $\psi$ and $\eta$ in \cref{eq:unconstraint} are set to 0.5 and 0.1, respectively.
Please refer to the \textit{supplementary material} for selecting $\psi$ and $\eta$ via grid search.
%The choice of hyperparameters in the experiment are validated in the next section.
\subsection{Comparison with SOTA methods}
In this section, we compare our DDFM with the state-of-the-art methods, including
the \textit{GAN-based generative methods group}:
FusionGAN~\cite{ma2019fusiongan},
GANMcC~\cite{DBLP:journals/tim/MaZSLX21},
TarDAL~\cite{DBLP:conf/cvpr/LiuFHWLZL22}, and
UMFusion~\cite{DBLP:conf/ijcai/WangLFL22};
and the \textit{discriminative methods group}:
U2Fusion~\cite{9151265},
RFNet~\cite{DBLP:conf/cvpr/Xu0YLL22}, and
DeFusion~\cite{Liang2022ECCV}.

\bfsection{Qualitative comparison}
We show the comparison of fusion results in \cref{fig:IVF2,fig:IVF3}. Our approach effectively combines thermal radiation information from infrared images with detailed texture information from visible images. As a result, objects located in dimly-lit environments are conspicuously accentuated, enabling easy distinguishing of foreground objects from the background. Moreover, previously indistinct background features due to low illumination now possess clearly defined edges and abundant contour information, enhancing our ability to comprehend the scene.

\bfsection{Quantitative comparison}
Subsequently, six metrics previously mentioned are utilized to quantitatively compare the fusion outcomes, as presented in \cref{tab:Quantitative}. Our method demonstrates remarkable performance across almost all metrics, affirming its suitability for different lighting and object categories.
Notably, the outstanding values for MI, VIF and Qabf across all datasets signify its ability to generate images that adhere to human visual perception while preserving the integrity of the source image information.

\begin{figure*}[t]
    \centering
    \includegraphics[width=0.9\linewidth]{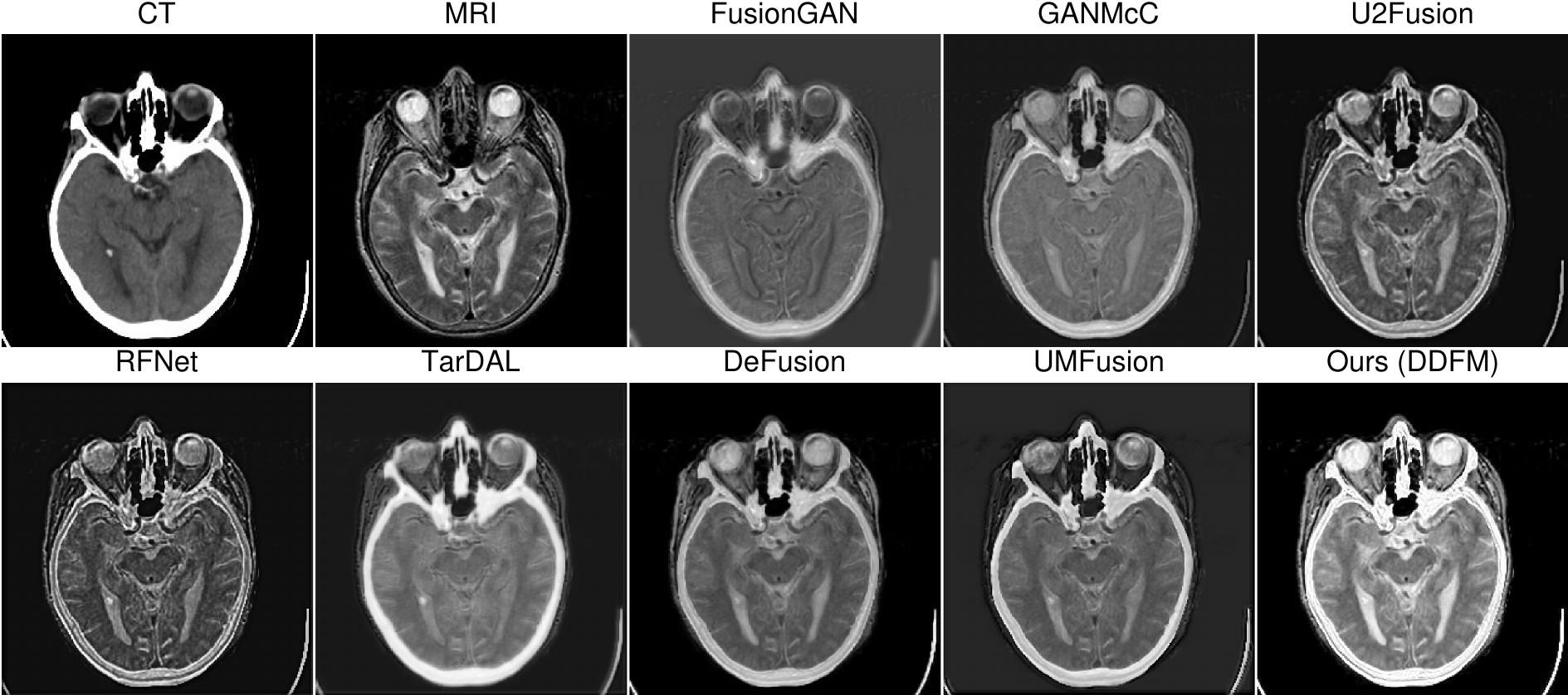}
    \caption{Visual comparison for MIF task.}
    \label{fig:MIF}
\end{figure*}

\subsection{Ablation studies}\label{sec:ablation}
Numerous ablation experiments are conducted to confirm the soundness of our various modules. The above six metrics are utilized to assess the fusion performance for the experimental groups, and results on the Roadscene testset are displayed in \cref{tab:ablation}.

\bfsection{Unconditional diffusion sampling module}
We first verify the effectiveness of DDPM. In Exp.~\uppercase\expandafter{\romannumeral1}, we eliminate the denoising diffusion generative framework, thus only the EM algorithm is employed to solve the optimization \cref{eq:lossfn0} and obtain the fusion image. In fairness, we keep the total iteration number consistent with DDFM.

\bfsection{EM module}
Next, we verify the components in the EM module. In Exp.~\uppercase\expandafter{\romannumeral2}, we removed the total variation penalty item $r(\boldsymbol{x})$ in \cref{eq:loss}. Then, we remove the Bayesian inference model. As mentioned earlier, $\phi$ in \cref{eq:lossfn0} can be automatically inferred in the hierarchical Bayesian model. Therefore, we manually set $\phi$ to 0.1 (Exp.~\uppercase\expandafter{\romannumeral3}) and 1 (Exp.~\uppercase\expandafter{\romannumeral4}), and used the ADMM algorithm to infer the model.

In conclusion, the results presented in Tab.~\ref{tab:ablation} demonstrate that none of the experimental groups is able to achieve fusion results comparable to our DDFM, further emphasizing the effectiveness and rationality of our approach.

\begin{table}[t]
    \centering
    \caption{The quantitative results of the MIF task, with the best and second-best values in \textbf{boldface} and \underline{underline}, respectively.}
    \label{tab:MIF}%
    \resizebox{\linewidth}{!}{
        \begin{tabular}{ccccccc}
            \toprule
            \multicolumn{7}{c}{\textbf{Dataset: Harvard Medical Fusion Dataset}~\cite{HarvardMIF}}                                   \\
            &  EN $\uparrow$   &   SD $\uparrow$   &  MI $\uparrow$   &  VIF $\uparrow$  & Qabf $\uparrow$  & SSIM $\uparrow$  \\ \midrule
            FGAN~\cite{ma2019fusiongan}       &       4.05       &       29.20       &       1.53       &       0.39       &       0.18       &       0.23       \\
            GMcC~\cite{DBLP:journals/tim/MaZSLX21} &       4.18       &       42.49       &       1.74       &       0.50       &       0.42       &       0.35       \\
            U2F~\cite{9151265}           &       4.14       &       48.89       &       1.80       &       0.50       &       0.55       &       1.14       \\
            RFN~\cite{DBLP:conf/cvpr/Xu0YLL22}   &  \textbf{4.75}   &       40.81       &       1.62       &       0.43       &       0.56       &       0.40       \\
            TarD~\cite{DBLP:conf/cvpr/LiuFHWLZL22} &       4.61       &       60.64       &       1.44       &       0.33       &       0.21       &       0.25       \\
            DeF~\cite{Liang2022ECCV}        &       4.21       & \underline{61.65} & \underline{1.85} & \underline{0.62} & \underline{0.59} & \underline{1.40} \\
            UMF~\cite{DBLP:conf/ijcai/WangLFL22}  &       4.61       &       27.28       &       1.62       &       0.40       &       0.27       &       0.30       \\
            Ours                  & \underline{4.64} &  \textbf{63.11}   &  \textbf{1.99}   &  \textbf{0.76}   &  \textbf{0.60}   &  \textbf{1.41}   \\ \bottomrule
        \end{tabular}
    }
\end{table}%

\section{Medical image fusion}\label{sec:experiment2}
In this section, MIF experiments are carried out to verify the effectiveness of our method.

%\subsection{Setup}
\bfsection{Setup}
We choose 50 pairs of medical images from the Harvard Medical Image Dataset~\cite{HarvardMIF} for the MIF experiments, including image pairs of MRI-CT, MRI-PET and MRI-SPECT. The generation strategy and evaluation metrics for the MIF task are identical to those used for IVF.

\bfsection{Comparison with SOTA methods}
Qualitative and quantitative results are shown in \cref{fig:MIF,tab:MIF}. It is evident that DDFM retains intricate textures while emphasizing structural information, leading to remarkable performance across both visual and almost all numerical metrics.

\section{Conclusion}
We propose DDFM, a novel generative image fusion algorithm based on the denoising diffusion probabilistic model (DDPM).
The generation problem is split into an unconditional DDPM to leverage image generative priors and a maximum likelihood sub-problem to preserve cross-modality information from source images.
We model the latter using a hierarchical Bayesian approach and its solution based on EM algorithm can be integrated into unconditional DDPM to accomplish conditional image fusion.
Experiments on infrared-visible and medical image fusion demonstrate that DDFM achieves promising fusion results.

\section*{Acknowledgement}\label{sec:6}
This work has been supported by the National Key Research and Development Program of China under grant 2018AAA0102201, the National Natural Science Foundation of China under Grant 61976174 and 12201497, the Macao Science and Technology Development Fund under Grant 061/2020/A2, Shaanxi Fundamental Science Research Project for Mathematics and Physics under Grant 22JSQ033, the Fundamental Research Funds for the Central Universities under Grant D5000220060, and partly supported by the Alexander von Humboldt Foundation.

{\small
    \bibliographystyle{ieee_fullname}
    \bibliography{xbib}

\begin{thebibliography}{10}\itemsep=-1pt

\bibitem{anderson1982reverse}
Brian~DO Anderson.
\newblock Reverse-time diffusion equation models.
\newblock {\em Stochastic Processes and their Applications}, 12(3):313--326,
  1982.

\bibitem{DBLP:journals/corr/abs-2004-10934}
Alexey Bochkovskiy, Chien{-}Yao Wang, and Hong{-}Yuan~Mark Liao.
\newblock Yolov4: Optimal speed and accuracy of object detection.
\newblock {\em CoRR}, abs/2004.10934, 2020.

\bibitem{DBLP:journals/corr/abs-2209-14687}
Hyungjin Chung, Jeongsol Kim, Michael~T. McCann, Marc~Louis Klasky, and
  Jong~Chul Ye.
\newblock Diffusion posterior sampling for general noisy inverse problems.
\newblock In {\em {ICLR}}, 2023.

\bibitem{DBLP:journals/pami/0002D21}
Xin Deng and Pier~Luigi Dragotti.
\newblock Deep convolutional neural network for multi-modal image restoration
  and fusion.
\newblock {\em {IEEE} Trans. Pattern Anal. Mach. Intell.}, 43(10):3333--3348,
  2021.

\bibitem{dhariwal2021diffusion}
Prafulla Dhariwal and Alexander Nichol.
\newblock Diffusion models beat gans on image synthesis.
\newblock {\em Advances in Neural Information Processing Systems},
  34:8780--8794, 2021.

\bibitem{DBLP:journals/tip/GaoDXXD22}
Fangyuan Gao, Xin Deng, Mai Xu, Jingyi Xu, and Pier~Luigi Dragotti.
\newblock Multi-modal convolutional dictionary learning.
\newblock {\em {IEEE} Trans. Image Process.}, 31:1325--1339, 2022.

\bibitem{DBLP:conf/nips/GoodfellowPMXWOCB14}
Ian~J. Goodfellow, Jean Pouget{-}Abadie, Mehdi Mirza, Bing Xu, David
  Warde{-}Farley, Sherjil Ozair, Aaron~C. Courville, and Yoshua Bengio.
\newblock Generative adversarial nets.
\newblock In {\em {NeurIPS}}, pages 2672--2680, 2014.

\bibitem{HarvardMIF}
{Harvard Medical website}.
\newblock \url{http://www.med.harvard.edu/AANLIB/home.html}.

\bibitem{he2023HQG}
Chunming He, Kai Li, Guoxia Xu, Jiangpeng Yan, Longxiang Tang, Yulun Zhang, Xiu
  Li, and Yaowei Wang.
\newblock Hqg-net: Unpaired medical image enhancement with high-quality
  guidance.
\newblock {\em arXiv preprint arXiv:2307.07829}, 2023.

\bibitem{He2023Camouflaged}
Chunming He, Kai Li, Yachao Zhang, Longxiang Tang, Yulun Zhang, Zhenhua Guo,
  and Xiu Li.
\newblock Camouflaged object detection with feature decomposition and edge
  reconstruction.
\newblock In {\em {CVPR}}, 2023.

\bibitem{he2023weaklysupervised}
Chunming He, Kai Li, Yachao Zhang, Guoxia Xu, Longxiang Tang, Yulun Zhang,
  Zhenhua Guo, and Xiu Li.
\newblock Weakly-supervised concealed object segmentation with sam-based pseudo
  labeling and multi-scale feature grouping.
\newblock {\em arXiv preprint arXiv:2305.11003}, 2023.

\bibitem{he2023strategic}
Chunming He, Kai Li, Yachao Zhang, Yulun Zhang, Zhenhua Guo, Xiu Li, Martin
  Danelljan, and Fisher Yu.
\newblock Strategic preys make acute predators: Enhancing camouflaged object
  detectors by generating camouflaged objects.
\newblock {\em arXiv preprint arXiv:2308.03166}, 2023.

\bibitem{DBLP:conf/nips/HoJA20}
Jonathan Ho, Ajay Jain, and Pieter Abbeel.
\newblock Denoising diffusion probabilistic models.
\newblock In {\em NeurIPS}, 2020.

\bibitem{huangreconet}
Zhanbo Huang, Jinyuan Liu, Xin Fan, Risheng Liu, Wei Zhong, and Zhongxuan Luo.
\newblock Reconet: Recurrent correction network for fast and efficient
  multi-modality image fusion.
\newblock In {\em {ECCV}}, 2022.

\bibitem{hyvarinen2005estimation}
Aapo Hyv{\"a}rinen and Peter Dayan.
\newblock Estimation of non-normalized statistical models by score matching.
\newblock {\em Journal of Machine Learning Research}, 6(4), 2005.

\bibitem{DBLP:journals/inffus/JamesD14}
Alex~Pappachen James and Belur~V. Dasarathy.
\newblock Medical image fusion: {A} survey of the state of the art.
\newblock {\em Inf. Fusion}, 19:4--19, 2014.

\bibitem{jiang2022towards}
Zhiying Jiang, Zengxi Zhang, Xin Fan, and Risheng Liu.
\newblock Towards all weather and unobstructed multi-spectral image stitching:
  Algorithm and benchmark.
\newblock In {\em {ACM MM}}, pages 3783--3791, 2022.

\bibitem{DBLP:journals/tip/JungKJHS20}
Hyungjoo Jung, Youngjung Kim, Hyunsung Jang, Namkoo Ha, and Kwanghoon Sohn.
\newblock Unsupervised deep image fusion with structure tensor representations.
\newblock {\em {IEEE} Trans. Image Process.}, 29:3845--3858, 2020.

\bibitem{kawar2022denoising}
Bahjat Kawar, Michael Elad, Stefano Ermon, and Jiaming Song.
\newblock Denoising diffusion restoration models.
\newblock {\em arXiv preprint arXiv:2201.11793}, 2022.

\bibitem{DBLP:journals/tip/LiMYZ20}
Hui Li, Kede Ma, Hongwei Yong, and Lei Zhang.
\newblock Fast multi-scale structural patch decomposition for multi-exposure
  image fusion.
\newblock {\em {IEEE} Trans. Image Process.}, 29:5805--5816, 2020.

\bibitem{DBLP:journals/tim/LiWD20}
Hui Li, Xiao{-}Jun Wu, and Tariq~S. Durrani.
\newblock Nestfuse: An infrared and visible image fusion architecture based on
  nest connection and spatial/channel attention models.
\newblock {\em {IEEE} Trans. Instrum. Meas.}, 69(12):9645--9656, 2020.

\bibitem{DBLP:journals/inffus/LiWK21}
Hui Li, Xiao{-}Jun Wu, and Josef Kittler.
\newblock Rfn-nest: An end-to-end residual fusion network for infrared and
  visible images.
\newblock {\em Inf. Fusion}, 73:72--86, 2021.

\bibitem{li2018densefuse}
Hui Li and Xiao-Jun Wu.
\newblock Densefuse: A fusion approach to infrared and visible images.
\newblock {\em {IEEE} Trans. Image Process.}, 28(5):2614--2623, 2018.

\bibitem{li2023lrrnet}
Hui Li, Tianyang Xu, Xiao-Jun Wu, Jiwen Lu, and Josef Kittler.
\newblock Lrrnet: A novel representation learning guided fusion network for
  infrared and visible images.
\newblock {\em IEEE transactions on pattern analysis and machine intelligence},
  2023.

\bibitem{Liang2022ECCV}
Pengwei Liang, Junjun Jiang, Xianming Liu, and Jiayi Ma.
\newblock Fusion from decomposition: A self-supervised decomposition approach
  for image fusion.
\newblock In {\em {ECCV}}, 2022.

\bibitem{DBLP:conf/cvpr/LiuFHWLZL22}
Jinyuan Liu, Xin Fan, Zhanbo Huang, Guanyao Wu, Risheng Liu, Wei Zhong, and
  Zhongxuan Luo.
\newblock Target-aware dual adversarial learning and a multi-scenario
  multi-modality benchmark to fuse infrared and visible for object detection.
\newblock In {\em {CVPR}}, pages 5792--5801. {IEEE}, 2022.

\bibitem{liu2020bilevel}
Risheng Liu, Jinyuan Liu, Zhiying Jiang, Xin Fan, and Zhongxuan Luo.
\newblock A bilevel integrated model with data-driven layer ensemble for
  multi-modality image fusion.
\newblock {\em {IEEE} Trans. Image Process.}, 30:1261--1274, 2020.

\bibitem{DBLP:conf/mm/LiuLL021}
Risheng Liu, Zhu Liu, Jinyuan Liu, and Xin Fan.
\newblock Searching a hierarchically aggregated fusion architecture for fast
  multi-modality image fusion.
\newblock In {\em {ACM MM}}, pages 1600--1608. {ACM}, 2021.

\bibitem{liu2023bi}
Zhu Liu, Jinyuan Liu, Guanyao Wu, Long Ma, Xin Fan, and Risheng Liu.
\newblock Bi-level dynamic learning for jointly multi-modality image fusion and
  beyond.
\newblock {\em arXiv preprint arXiv:2305.06720}, 2023.

\bibitem{ma2016infrared}
Jiayi Ma, Chen Chen, Chang Li, and Jun Huang.
\newblock Infrared and visible image fusion via gradient transfer and total
  variation minimization.
\newblock {\em Inf. Fusion}, 31:100--109, 2016.

\bibitem{ma2020infrared}
Jiayi Ma, Pengwei Liang, Wei Yu, Chen Chen, Xiaojie Guo, Jia Wu, and Junjun
  Jiang.
\newblock Infrared and visible image fusion via detail preserving adversarial
  learning.
\newblock {\em Inf. Fusion}, 54:85--98, 2020.

\bibitem{ma2019infrared}
Jiayi Ma, Yong Ma, and Chang Li.
\newblock Infrared and visible image fusion methods and applications: A survey.
\newblock {\em Inf. Fusion}, 45:153--178, 2019.

\bibitem{DBLP:journals/tip/MaXJMZ20}
Jiayi Ma, Han Xu, Junjun Jiang, Xiaoguang Mei, and Xiao{-}Ping~(Steven) Zhang.
\newblock Ddcgan: {A} dual-discriminator conditional generative adversarial
  network for multi-resolution image fusion.
\newblock {\em {IEEE} Trans. Image Process.}, 29:4980--4995, 2020.

\bibitem{ma2019fusiongan}
Jiayi Ma, Wei Yu, Pengwei Liang, Chang Li, and Junjun Jiang.
\newblock Fusiongan: A generative adversarial network for infrared and visible
  image fusion.
\newblock {\em Inf. Fusion}, 48:11--26, 2019.

\bibitem{DBLP:journals/tim/MaZSLX21}
Jiayi Ma, Hao Zhang, Zhenfeng Shao, Pengwei Liang, and Han Xu.
\newblock Ganmcc: {A} generative adversarial network with multiclassification
  constraints for infrared and visible image fusion.
\newblock {\em {IEEE} Trans. Instrum. Meas.}, 70:1--14, 2021.

\bibitem{mao2017least}
Xudong Mao, Qing Li, Haoran Xie, Raymond~YK Lau, Zhen Wang, and Stephen
  Paul~Smolley.
\newblock Least squares generative adversarial networks.
\newblock In {\em {ICCV}}, pages 2794--2802, 2017.

\bibitem{meher2019a}
Bikash Meher, Sanjay Agrawal, Rutuparna Panda, and Ajith Abraham.
\newblock A survey on region based image fusion methods.
\newblock {\em Inf. Fusion}, 48:119--132, 2019.

\bibitem{mirza2014conditional}
Mehdi Mirza and Simon Osindero.
\newblock Conditional generative adversarial nets.
\newblock {\em arXiv preprint arXiv:1411.1784}, 2014.

\bibitem{DBLP:conf/icml/NicholD21}
Alexander~Quinn Nichol and Prafulla Dhariwal.
\newblock Improved denoising diffusion probabilistic models.
\newblock In {\em {ICML}}, pages 8162--8171, 2021.

\bibitem{qin2023diverse}
Haotong Qin, Yifu Ding, Xiangguo Zhang, Jiakai Wang, Xianglong Liu, and Jiwen
  Lu.
\newblock Diverse sample generation: Pushing the limit of generative data-free
  quantization.
\newblock {\em IEEE Transactions on Pattern Analysis and Machine Intelligence},
  2023.

\bibitem{qin2023bibench}
Haotong Qin, Mingyuan Zhang, Yifu Ding, Aoyu Li, Ziwei Liu, Fisher Yu, and
  Xianglong Liu.
\newblock Bibench: Benchmarking and analyzing network binarization.
\newblock In {\em {ICML}}, 2023.

\bibitem{qin2022distribution}
Haotong Qin, Xiangguo Zhang, Ruihao Gong, Yifu Ding, Yi Xu, and Xianglong Liu.
\newblock Distribution-sensitive information retention for accurate binary
  neural network.
\newblock {\em International Journal of Computer Vision}, 2022.

\bibitem{DBLP:conf/cvpr/QinZHGDJ19}
Xuebin Qin, Zichen~Vincent Zhang, Chenyang Huang, Chao Gao, Masood Dehghan, and
  Martin J{\"{a}}gersand.
\newblock Basnet: Boundary-aware salient object detection.
\newblock In {\em {CVPR}}, pages 7479--7489. {CVF} / {IEEE}, 2019.

\bibitem{rombach2022high}
Robin Rombach, Andreas Blattmann, Dominik Lorenz, Patrick Esser, and Bj{\"o}rn
  Ommer.
\newblock High-resolution image synthesis with latent diffusion models.
\newblock In {\em {CVPR}}, pages 10684--10695, 2022.

\bibitem{DBLP:journals/ijcv/RussakovskyDSKS15}
Olga Russakovsky, Jia Deng, Hao Su, Jonathan Krause, Sanjeev Satheesh, Sean Ma,
  Zhiheng Huang, Andrej Karpathy, Aditya Khosla, Michael~S. Bernstein,
  Alexander~C. Berg, and Li Fei{-}Fei.
\newblock Imagenet large scale visual recognition challenge.
\newblock {\em Int. J. Comput. Vis.}, 115(3):211--252, 2015.

\bibitem{DBLP:conf/iclr/SongME21}
Jiaming Song, Chenlin Meng, and Stefano Ermon.
\newblock Denoising diffusion implicit models.
\newblock In {\em {ICLR}}, 2021.

\bibitem{song2023pseudoinverse}
Jiaming Song, Arash Vahdat, Morteza Mardani, and Jan Kautz.
\newblock Pseudoinverse-guided diffusion models for inverse problems.
\newblock In {\em {ICLR}}, 2023.

\bibitem{song2019generative}
Yang Song and Stefano Ermon.
\newblock Generative modeling by estimating gradients of the data distribution.
\newblock {\em Advances in Neural Information Processing Systems}, 32, 2019.

\bibitem{DBLP:conf/iclr/0011SKKEP21}
Yang Song, Jascha Sohl{-}Dickstein, Diederik~P. Kingma, Abhishek Kumar, Stefano
  Ermon, and Ben Poole.
\newblock Score-based generative modeling through stochastic differential
  equations.
\newblock In {\em {ICLR}}, 2021.

\bibitem{DBLP:journals/inffus/TangYM22}
Linfeng Tang, Jiteng Yuan, and Jiayi Ma.
\newblock Image fusion in the loop of high-level vision tasks: {A}
  semantic-aware real-time infrared and visible image fusion network.
\newblock {\em Inf. Fusion}, 82:28--42, 2022.

\bibitem{DBLP:journals/inffus/TangYZJM22}
Linfeng Tang, Jiteng Yuan, Hao Zhang, Xingyu Jiang, and Jiayi Ma.
\newblock Piafusion: {A} progressive infrared and visible image fusion network
  based on illumination aware.
\newblock {\em Inf. Fusion}, 83-84:79--92, 2022.

\bibitem{TNO}
Alexander Toet and Maarten~A. Hogervorst.
\newblock {Progress in color night vision}.
\newblock {\em Optical Engineering}, 51(1):1 -- 20, 2012.

\bibitem{DBLP:journals/corr/abs-2107-09011}
Vibashan VS, Jeya Maria~Jose Valanarasu, Poojan Oza, and Vishal~M. Patel.
\newblock Image fusion transformer.
\newblock {\em CoRR}, abs/2107.09011, 2021.

\bibitem{DBLP:conf/ijcai/WangLFL22}
Di Wang, Jinyuan Liu, Xin Fan, and Risheng Liu.
\newblock Unsupervised misaligned infrared and visible image fusion via
  cross-modality image generation and registration.
\newblock In {\em {IJCAI}}, pages 3508--3515. ijcai.org, 2022.

\bibitem{wang2021dual}
Jiakai Wang, Aishan Liu, Zixin Yin, Shunchang Liu, Shiyu Tang, and Xianglong
  Liu.
\newblock Dual attention suppression attack: Generate adversarial camouflage in
  physical world.
\newblock In {\em {CVPR}}, pages 8565--8574, 2021.

\bibitem{wang2022defensive}
Jiakai Wang, Zixin Yin, Pengfei Hu, Aishan Liu, Renshuai Tao, Haotong Qin,
  Xianglong Liu, and Dacheng Tao.
\newblock Defensive patches for robust recognition in the physical world.
\newblock In {\em {CVPR}}, pages 2456--2465, 2022.

\bibitem{DBLP:conf/iclr/XiaoKV22}
Zhisheng Xiao, Karsten Kreis, and Arash Vahdat.
\newblock Tackling the generative learning trilemma with denoising diffusion
  gans.
\newblock In {\em {ICLR}}, 2022.

\bibitem{9151265}
Han Xu, Jiayi Ma, Junjun Jiang, Xiaojie Guo, and Haibin Ling.
\newblock U2fusion: {A} unified unsupervised image fusion network.
\newblock {\em {IEEE} Trans. Pattern Anal. Mach. Intell.}, 44(1):502--518,
  2022.

\bibitem{xu2020aaai}
Han Xu, Jiayi Ma, Zhuliang Le, Junjun Jiang, and Xiaojie Guo.
\newblock Fusiondn: A unified densely connected network for image fusion.
\newblock In {\em {AAAI} Conference on Artificial Intelligence, {AAAI}}, pages
  12484--12491, 2020.

\bibitem{DBLP:conf/cvpr/Xu0YLL22}
Han Xu, Jiayi Ma, Jiteng Yuan, Zhuliang Le, and Wei Liu.
\newblock Rfnet: Unsupervised network for mutually reinforcing multi-modal
  image registration and fusion.
\newblock In {\em {CVPR}}, pages 19647--19656. {IEEE}, 2022.

\bibitem{xu2023murf}
Han Xu, Jiteng Yuan, and Jiayi Ma.
\newblock Murf: Mutually reinforcing multi-modal image registration and fusion.
\newblock {\em IEEE Transactions on Pattern Analysis and Machine Intelligence},
  2023.

\bibitem{DBLP:conf/cvpr/Xu0ZSL021}
Shuang Xu, Jiangshe Zhang, Zixiang Zhao, Kai Sun, Junmin Liu, and Chunxia
  Zhang.
\newblock Deep gradient projection networks for pan-sharpening.
\newblock In {\em {CVPR}}, pages 1366--1375. {CVF} / {IEEE}, 2021.

\bibitem{DBLP:journals/corr/abs-2005-08448}
Shuang Xu, Zixiang Zhao, Yicheng Wang, Chunxia Zhang, Junmin Liu, and Jiangshe
  Zhang.
\newblock Deep convolutional sparse coding networks for image fusion.
\newblock {\em CoRR}, abs/2005.08448, 2020.

\bibitem{wang2022ddnm}
Wang Yinhuai, Yu Jiwen, and Zhang Jian.
\newblock Zero shot image restoration using denoising diffusion null-space
  model.
\newblock {\em arXiv:2212.00490}, 2022.

\bibitem{DBLP:journals/ijcv/ZhangM21}
Hao Zhang and Jiayi Ma.
\newblock Sdnet: {A} versatile squeeze-and-decomposition network for real-time
  image fusion.
\newblock {\em Int. J. Comput. Vis.}, 129(10):2761--2785, 2021.

\bibitem{DBLP:conf/aaai/ZhangXXGM20}
Hao Zhang, Han Xu, Yang Xiao, Xiaojie Guo, and Jiayi Ma.
\newblock Rethinking the image fusion: {A} fast unified image fusion network
  based on proportional maintenance of gradient and intensity.
\newblock In {\em {AAAI}}, pages 12797--12804. {AAAI} Press, 2020.

\bibitem{zhang2021deep}
Xingchen Zhang.
\newblock Deep learning-based multi-focus image fusion: A survey and a
  comparative study.
\newblock {\em IEEE Transactions on Pattern Analysis and Machine Intelligence},
  2021.

\bibitem{DBLP:journals/inffus/ZhangLSYZZ20}
Yu Zhang, Yu Liu, Peng Sun, Han Yan, Xiaolin Zhao, and Li Zhang.
\newblock {IFCNN:} {A} general image fusion framework based on convolutional
  neural network.
\newblock {\em Inf. Fusion}, 54:99--118, 2020.

\bibitem{DBLP:journals/corr/abs-2211-14461}
Zixiang Zhao, Haowen Bai, Jiangshe Zhang, Yulun Zhang, Shuang Xu, Zudi Lin,
  Radu Timofte, and Luc Van~Gool.
\newblock Cddfuse: Correlation-driven dual-branch feature decomposition for
  multi-modality image fusion.
\newblock In {\em {CVPR}}, pages 5906--5916, June 2023.

\bibitem{zhao2023equivariant}
Zixiang Zhao, Haowen Bai, Jiangshe Zhang, Yulun Zhang, Kai Zhang, Shuang Xu,
  Dongdong Chen, Radu Timofte, and Luc Van~Gool.
\newblock Equivariant multi-modality image fusion.
\newblock {\em arXiv preprint arXiv:2305.11443}, 2023.

\bibitem{ZHAO2020107734}
Zixiang Zhao, Shuang Xu, Chunxia Zhang, Junmin Liu, and Jiangshe Zhang.
\newblock Bayesian fusion for infrared and visible images.
\newblock {\em Signal Processing}, 177, 2020.

\bibitem{zhaoijcai2020}
Zixiang Zhao, Shuang Xu, Chunxia Zhang, Junmin Liu, Jiangshe Zhang, and Pengfei
  Li.
\newblock {DIDFuse}: Deep image decomposition for infrared and visible image
  fusion.
\newblock In {\em {IJCAI}}, pages 970--976, 2020.

\bibitem{DBLP:journals/tcsv/ZhaoXZLZL22}
Zixiang Zhao, Shuang Xu, Jiangshe Zhang, Chengyang Liang, Chunxia Zhang, and
  Junmin Liu.
\newblock Efficient and model-based infrared and visible image fusion via
  algorithm unrolling.
\newblock {\em {IEEE} Trans. Circuits Syst. Video Technol.}, 32(3):1186--1196,
  2022.

\bibitem{zhao2023spherical}
Zixiang Zhao, Jiangshe Zhang, Xiang Gu, Chengli Tan, Shuang Xu, Yulun Zhang,
  Radu Timofte, and Luc Van~Gool.
\newblock Spherical space feature decomposition for guided depth map
  super-resolution.
\newblock In {\em {ICCV}}, 2023.

\bibitem{DBLP:journals/corr/abs-2104-06977}
Zixiang Zhao, Jiangshe Zhang, Shuang Xu, Zudi Lin, and Hanspeter Pfister.
\newblock Discrete cosine transform network for guided depth map
  super-resolution.
\newblock In {\em {CVPR}}, pages 5697--5707, June 2022.

\bibitem{DBLP:conf/icmcs/00010XSHL021}
Zixiang Zhao, Jiangshe Zhang, Shuang Xu, Kai Sun, Lu Huang, Junmin Liu, and
  Chunxia Zhang.
\newblock {FGF-GAN:} {A} lightweight generative adversarial network for
  pansharpening via fast guided filter.
\newblock In {\em {ICME}}, pages 1--6. {IEEE}, 2021.

\end{thebibliography}
}

\end{document}